\documentclass{article}
\IfFileExists{neurips_2026.sty}{\usepackage[nonatbib,preprint]{neurips_2026}}{\usepackage[margin=1in]{geometry}}

\usepackage[utf8]{inputenc}
\usepackage[T1]{fontenc}
\usepackage{hyperref}
\hypersetup{hidelinks}
\usepackage{url}
\usepackage{booktabs}
\usepackage{amsfonts}
\usepackage{amsmath}
\usepackage{amssymb}
\usepackage{amsthm}
\usepackage{nicefrac}
\usepackage{microtype}
\usepackage{xcolor}
\usepackage{graphicx}
\graphicspath{{figures/}{./}}
\usepackage{subcaption}
\usepackage{float}
\usepackage{tikz}
\usetikzlibrary{arrows.meta,positioning,calc,fit,shapes.geometric}
\usepackage{algorithm}
\usepackage{algpseudocode}
\usepackage{comment}
\usepackage{enumitem}

\usetikzlibrary{backgrounds}

\newtheorem{theorem}{Theorem}
\newtheorem{lemma}[theorem]{Lemma}
\newtheorem{proposition}[theorem]{Proposition}
\newtheorem{corollary}[theorem]{Corollary}
\newtheorem{definition}[theorem]{Definition}
\newtheorem{remark}[theorem]{Remark}
\newtheorem{assumption}[theorem]{Assumption}

\newcommand{\bz}{\mathbf{z}}
\newcommand{\bx}{\mathbf{x}}
\newcommand{\bd}{\mathbf{d}}
\newcommand{\be}{\mathbf{e}}
\newcommand{\ba}{\mathbf{a}}

\newcommand{\bmu}{\boldsymbol{\mu}}
\newcommand{\bSigma}{\boldsymbol{\Sigma}}
\newcommand{\btheta}{\boldsymbol{\theta}}
\newcommand{\cC}{\mathcal{C}}

\newcommand{\cZ}{\mathcal{Z}}
\newcommand{\cY}{\mathcal{Y}}
\newcommand{\E}{\mathbb{E}}
\newcommand{\KL}{\mathrm{KL}}

\newcommand{\argmin}{\operatorname{arg\,min}}
\newcommand{\argmax}{\operatorname{arg\,max}}
\newcommand{\Cenc}{\cC^{\mathrm{enc}}}
\newcommand{\Cdec}{\cC^{\mathrm{dec}}}
\newcommand{\Ce}{C^{\mathrm{enc}}}
\newcommand{\Cd}{C^{\mathrm{dec}}}
\newcommand{\Ped}{P_{\mathrm{ed}}}
\newcommand{\Ked}{\mathsf{K}_{e\to d}}
\newcommand{\Icode}{\mathcal{I}_{\mathrm{code}}}
\newcommand{\Reff}{R_{\mathrm{eff}}}
\newcommand{\nISI}{\mathrm{nISI}}

\newcommand{\HATATWO}{79.76}
\newcommand{\ARITWO}{0.5990}
\newcommand{\EPOCHLONG}{196{,}200}

\newcommand{\HATAFMNIST}{72.8}
\newcommand{\ARIFMNIST}{0.52}
\newcommand{\HATACIFAR}{41.3}
\newcommand{\ARICIFAR}{0.21}

\newcommand{\LVAL}{0.29}

\newcommand{\GAMMAVAL}{0.14}
\newcommand{\EDELTA}{1.13}
\newcommand{\EDELTAAGG}{36.8}
\newcommand{\CHECKBATCHSIZE}{32}
\newcommand{\RESIDUALAGG}{0.67}

\newcommand{\KAPPAMISMATCHTERM}{6.03}
\newcommand{\ETAP}{0.122}
\newcommand{\ETAESS}{78.4}
\newcommand{\ETAESSMIN}{42}
\newcommand{\ETAESSMAX}{99}
\newcommand{\ETACILOW}{0.115}
\newcommand{\ETACIHIGH}{0.129}
\newcommand{\BOUNDNUMERIC}{0.778}

\newcommand{\FASTCONV}{5{,}000}

\newcommand{\Kq}{\mathsf{K}_q}
\newcommand{\Kp}{\mathsf{K}_p}
\newcommand{\Gammaq}{\Gamma_q}
\newcommand{\Gammap}{\Gamma_p}
\newcommand{\barq}{\bar q}
\newcommand{\barp}{\bar p}
\newcommand{\Amu}{\mathcal A_{\mu}}
\newcommand{\Aq}{\mathcal A_{q}}
\newcommand{\Ap}{\mathcal A_{p}}

\title{Lost and Found in Translation:\\Variational Diagnostics for Neural Codebook Channels}

\author{%
  Yusuke Hayashi\thanks{Also affiliated with AI Alignment Network and Humanity Brain.} \\
  Artificial Life Institute \\
  \texttt{yusuke.hayashi@alife.institute}
}
\date{}

\begin{document}
\sloppy
\maketitle

\begin{abstract}
Classical communication systems fail not only through random noise
but also when transmitter and receiver use incompatible operational
codebooks.
Variational autoencoders (VAEs) train an encoder $q_\phi$ and decoder
$p_\theta$ jointly, and practitioners treat the resulting latent space
as a discrete code---for clustering, conditional generation, and
mechanistic interpretability.
Yet standard VAE diagnostics---ELBO, active units, mutual information, and
code histograms---certify only whether this code is \emph{used}, never
whether the decoder \emph{reads} each latent under the encoder's code.
We close this gap with the \emph{neural codebook channel} $\Ked(j\mid i)$,
a coupled encoder--decoder diagnostic whose off-diagonal mass is bounded
by an architecture-free Bernoulli-KL certificate
$d_{\rm bin}(1-\mathcal A\,\|\,\bar\eta_p)\le\bar\Delta$
controlled by the variational gap.
The certificate is the operational specialization of the classical KL
chain rule under disintegration to the encoder--decoder disagreement
event; it is complemented by a constructive marginal-impossibility
result---no combination of marginal histograms, entropies, active-code
counts, or mutual information determines $\Ked$, so $\Ked$ is provably
non-redundant with existing summaries.
We audit the certificate on four sklearn datasets (finite-grid exact,
$5/5$ seeds, $20/20$ pairs satisfy the bound), a 2D model where the
bound is non-vacuous at $2.71\!\times$ the observed disagreement and
the four-term identity closes within $10^{-4}$, MNIST under
importance-sampling control, and a VQ-VAE attaining the predicted limit
$\hat{\mathcal A}{=}1.000$.
The package $(\Ked,\,\mathcal A,\,R_{\rm eff},\,R,\,\mathrm{AU})$ is an
audit-ready reporting unit that complements existing VAE diagnostics by
closing the coupled question they leave open.
More broadly, the framework makes \emph{mismatched decoding}---a failure
mode classical communication theory named decades ago---visible inside a
single deep generative model, with transport to other code-like
continuous representations as a scope roadmap rather than audited
evidence.
\end{abstract}

\section{Introduction}
\label{sec:intro}

Variational autoencoders~\cite{kingma2014auto,rezende2014stochastic} train an
encoder $q_\phi(z\mid x)$ and a decoder $p_\theta(x\mid z)$ jointly, and the
resulting latent space is routinely treated as a discrete code: clusters for
taxonomy and retrieval, anchor points for conditional generation, feature axes
for representation comparison, and probe targets for mechanistic
interpretability~\cite{higgins2017beta,chen2018isolating,locatello2019challenging,khemakhem2020ivae,moschella2022relative,fumero2024latentfunctional,huh2024platonic,elhage2022superposition,cunningham2023sparse}.
Whenever the latent space is read in this way, the practitioner is implicitly
asking a paired question: \emph{does the decoder interpret a region of latent
space under the same code that the encoder used to place mass there?}
Classical communication theory has a name for the failure mode in which it
does not---mismatched decoding~\cite{scarlett2020mismatched,farvardin1990vq,proakis2008digital}---and
prior VAE work has identified an encoder--decoder consistency
problem~\cite{cemgil2020avae,wang2020coupled} or detected the special
case of posterior collapse from the marginal
side~\cite{dang2024posterior}; what is missing, to our knowledge, is a
\emph{coupled-channel finite-codebook diagnostic} whose off-diagonal mass is
audit-bounded by the variational gap.

\begin{figure}[t]
\centering
\IfFileExists{figures/fig1_alignment_snapshots.png}{%
\includegraphics[width=\linewidth]{figures/fig1_alignment_snapshots.png}%
}{\IfFileExists{figures/fig1_alignment_snapshots.pdf}{%
\includegraphics[width=\linewidth]{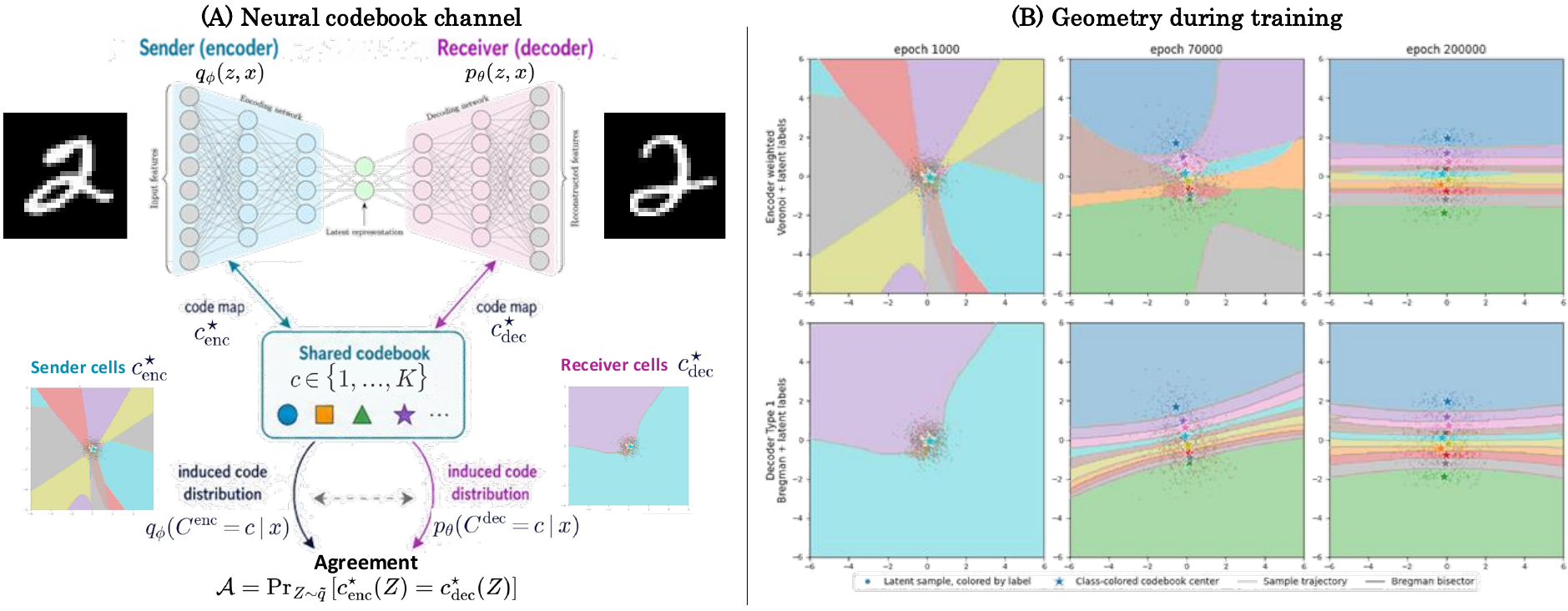}%
}{%
\begin{tikzpicture}[x=1cm,y=1cm,>=Latex]
  \node[draw,rounded corners,minimum width=2.7cm,minimum height=0.95cm] (enc) at (0,0) {encoder code};
  \node[draw,rounded corners,minimum width=2.7cm,minimum height=0.95cm] (z) at (3.6,0) {shared latent $Z$};
  \node[draw,rounded corners,minimum width=2.7cm,minimum height=0.95cm] (dec) at (7.2,0) {decoder reading};
  \draw[->] (enc) -- node[above]{label $i$} (z);
  \draw[->] (z) -- node[above]{label $j$} (dec);
  \node[draw,rounded corners,align=center,minimum width=3.1cm,minimum height=0.95cm] (ked) at (3.6,-1.55) {$\Ked(j|i)$\\coupled channel};
  \draw[->] (z) -- (ked);
  \node[draw,rounded corners,align=center,minimum width=2.4cm,minimum height=0.7cm] (diag) at (1.15,-2.7) {diagonal mass\\agreement};
  \node[draw,rounded corners,align=center,minimum width=2.4cm,minimum height=0.7cm] (off) at (5.85,-2.7) {off-diagonal mass\\interference};
  \draw[->] (ked) -- (diag);
  \draw[->] (ked) -- (off);
\end{tikzpicture}%
}}
\caption{\textbf{Codebook Agreement asks whether one shared latent draw receives the same encoder and decoder label.} \textbf{(A)} Schematic of the neural codebook channel: encoder and decoder each induce an operational code map on a shared latent $Z$, and the joint table $\Ped(i,j)$ together with the row-normalized channel $\Ked(j\mid i)$ records how encoder codes are read by the decoder. Marginal-only diagnostics cannot recover this table (Proposition~\ref{prop:marginal-impossibility}); the variational gap controls its binary off-diagonal specialization (Corollary~\ref{cor:agreement}). \textbf{(B)} Training-time evolution of latent codebook geometry on Setting~1-long at epochs $1000,70000,200000$: encoder additively-weighted Voronoi cells (top) and decoder Type-1 Bregman cells (bottom) progressively align. Scalar diagnostics $(\mathcal{A}, \Icode, \Reff/\log K)$ at the same checkpoints are $(0.245,0.755,0.290),(0.548,0.452,0.560),(0.980,0.020,0.959)$; full $\Ked$ heatmaps in Appendix Figure~\ref{fig:codebook-channel-heatmaps-full}. The sequence is a visual diagnostic, not a finite-grid certificate; proof-backed claims are in Tables~\ref{tab:cross-dataset-summary}--\ref{tab:finite-grid-audit-main}.}
\label{fig:teaser}
\end{figure}

\paragraph{What existing diagnostics measure, and what they miss.}
Standard VAE evaluation reports ELBO, reconstruction error, rate, active units
(AU), marginal code histograms, mutual information, and disentanglement
scores~\cite{alemi2018fixing,burda2016importance,higgins2017beta,chen2018isolating,tishby2000information,cover2006elements}.
Each of these is a marginal or aggregate quantity---a count of \emph{usage},
a sum over the channel. None of them inspects the \emph{coupling} between the
encoder's code map and the decoder's reading map. The omission has a concrete
cost: a decoder may permute two encoder regions, merge several encoder regions
into one reading, or agree trivially because both sides collapsed onto a
single code, and marginal usage statistics will look healthy or even excellent
in each case unless paired with non-collapse diagnostics. The information-theoretic
template for this kind of failure---transmitter and receiver using incompatible
codebooks~\cite{shannon1948,scarlett2020mismatched,farvardin1990vq}---makes the
gap precise: the missing diagnostic is the joint encoder--decoder channel,
not any of its marginals.

\paragraph{The diagnostic object and the marginal-impossibility result.}
We propose a single object as the missing reporting unit. Given measurable
encoder and decoder hard-code maps $c^{\star}_{\rm enc},c^{\star}_{\rm dec}$
on a $K$-symbol alphabet, draw one latent sample $Z$ and record both labels.
The joint table $\Ped(i,j)=\Pr[c^{\star}_{\rm enc}(Z)=i,\,c^{\star}_{\rm dec}(Z)=j]$
and the row-normalized \emph{neural codebook channel} $\Ked(j\mid i)$ together
record how encoder codes are read by the decoder; we use the word
``semantics'' only in this operational sense, not as human-interpretable
meaning. Diagonal mass is \emph{Codebook Agreement} $\mathcal A$; off-diagonal
mass is inter-code interference (Figure~\ref{fig:teaser}). The
\textbf{marginal-impossibility result} (Proposition~\ref{prop:marginal-impossibility})
is a constructive impossibility: for every $K\geq 2$, two coupled channels
can share identical encoder marginals, decoder marginals, entropies,
active-code counts, and mutual information yet have opposite raw agreement.
Thus $\Ked$ is not derivable from any combination of standard scalar or
marginal diagnostics, and reporting it is provably non-redundant with
existing summaries. The result also recovers the permutation indeterminacy
familiar from the latent-identifiability
literature~\cite{hyvarinen2019nonlinear,khemakhem2020ivae,locatello2019challenging}
as a concrete failure mode of marginal VAE diagnostics rather than as an
asymptotic obstruction.

\paragraph{Auditing $\Ked$ through the variational gap.}
The same diagnostic is auditable through the VAE training loss. Applying a
standard KL chain rule under disintegration~\cite{polyanskiy2024information,cover2006elements}
to the binary disagreement event $E=\{c^{\star}_{\rm enc}\neq c^{\star}_{\rm dec}\}$
decomposes the average variational gap into three non-negative terms
(Lemma~\ref{thm:universal}, Corollary~\ref{thm:codebook-identity}):
\begin{equation*}
\bar\Delta \;=\; \underbrace{d_{\rm bin}(1-\mathcal A\|\bar\eta_p)}_{\text{code-level}}
\;+\; \underbrace{\mathcal J_E}_{\text{Jensen residual}}
\;+\; \underbrace{\bar\rho_E}_{\text{within-cell residual}}.
\end{equation*}
Whenever $\bar\Delta$ is exactly enumerated or otherwise controlled, the
off-diagonal mass $1-\mathcal A$ is constrained by a one-dimensional
Bernoulli-KL inequality (Corollary~\ref{cor:agreement}); the residuals
$\mathcal J_E,\bar\rho_E$ are themselves diagnostics, measuring respectively
how much per-example mismatch cancels in the average and how much mismatch
the chosen statistic hides inside its cells. The decomposition itself is a
classical identity; the contribution is the operational specialization that
turns it into a reusable audit layer for neural codebooks. We instantiate
$c^{\star}_{\rm enc}$ as an additively-weighted Voronoi rule from the
GMM-summarized aggregate posterior~\cite{aurenhammer1987power,boissonnat2008mahalanobis}
and $c^{\star}_{\rm dec}$ as a Type-1 Bregman Voronoi rule in decoder-output
space~\cite{banerjee2005clustering,nielsen2010bregman}, then audit the
certificate on five-seed finite-grid exact posteriors across four
low-dimensional datasets and on $28\!\times\!28$ MNIST in an
importance-sampling-bounded regime; every dataset--seed pair on which the
certificate can be evaluated satisfies it (Tables~\ref{tab:cross-dataset-summary}--\ref{tab:mnist-audited}).

\paragraph{What this changes.}
The framework reframes VAE latent analysis from asking only whether codes are
\emph{used} to asking whether the encoder's codes are \emph{read consistently}
by the decoder, and supplies a portable, audit-ready reporting unit---the
package $(\Ked,\mathcal A,\Reff,R,\mathrm{AU})$ with explicit
certified/audited/heuristic labels---for any continuous latent space later
treated as a discrete code. The coupled question is the one downstream
clustering, retrieval, conditional generation, and interpretability tasks
actually depend on; existing marginal diagnostics can certify channel usage
without certifying that usage is consistent. Our scope is correspondingly
narrow: we do not claim state-of-the-art generation, a training-dynamics
theorem, a continuous-law certificate without quadrature control, or that
high agreement alone evidences emergent symbols. The claim is that when a VAE
latent space is treated as a code, the coupled encoder--decoder channel is
the missing diagnostic object, and the variational gap specifies when its
binary mismatch certificate is mathematically licensed.


\paragraph{Contributions.}
\begin{enumerate}[leftmargin=*]
\item \textbf{Coupled diagnostic.} We introduce Codebook Agreement and the
neural codebook channel $\Ked$, and prove a marginal-impossibility result:
standard marginal summaries, including histograms, entropies, active-code
counts, and mutual information, do not determine how encoder codes are read by
the decoder.
\item \textbf{Variational audit certificate.} Specializing the classical KL
chain rule to the encoder--decoder disagreement event decomposes the average
variational gap into a Bernoulli-KL code-level term, a Jensen residual, and a
within-cell residual, yielding a one-dimensional certificate for off-diagonal
disagreement whenever the gap and reference posterior rate are controlled.
\item \textbf{Instantiations and audits.} We instantiate the diagnostic with
Voronoi/Bregman code maps for common VAE architectures, report finite-grid
exact audit results, an MNIST audited regime, a Setting~1 audit where the
bound is non-vacuous ($2.71\!\times$ slack) and the four-term decomposition
closes within $10^{-4}$, and a VQ-VAE on MNIST attaining $\hat{\mathcal A}{=}1.000$
(Tier-3 endpoint of Proposition~\ref{prop:spectrum}); we include a
sensitivity sweep over code granularity to make the code-map dependence
explicit.
\end{enumerate}

\section{The Neural Codebook Channel}
\label{sec:agreement}

This section defines the diagnostic object the rest of the paper studies:
the coupled encoder--decoder channel $\Ked$, together with its scalar summary
$\mathcal A$ (Codebook Agreement) and its mutual-information summary $\Reff$
(effective code rate). The definitions are operational: they presume nothing
about whether the VAE was trained to expose discrete structure, only that
encoder and decoder each induce a measurable hard-code map on latent space.

\subsection{Encoder and decoder code maps}
\label{sec:operational-maps}

A codebook in this paper is an \emph{operational statistic}, not a primitive
discrete latent variable. Let
\[
  c^{\star}_{\rm enc}:\cZ\to\{1,\ldots,K\},\qquad
  c^{\star}_{\rm dec}:\cZ\to\{1,\ldots,K\}
\]
be measurable encoder-side and decoder-side hard-code maps on latent space.
Concrete instances appear in \S\ref{sec:theory}: $c^{\star}_{\rm enc}$ from
the GMM-summarized aggregate posterior, $c^{\star}_{\rm dec}$ from a Bregman
nearest-prototype rule in decoder-output space. The variational certificate of
\S\ref{sec:universal} requires only measurability and absolute continuity
$q_\phi(\cdot|x)\ll p_\theta(\cdot|x)$; it does not require either map to be
trained or even smooth.

\subsection{Codebook agreement and the coupled channel}
\label{sec:agreement-defs}

The diagnostic question this paper asks is whether a single shared latent
draw is given the same label by encoder and decoder.

\begin{definition}[Codebook Agreement]
\label{def:agreement}
With $\bar q:=\E_{\bx}q_\phi(\cdot|\bx)$ the encoder aggregate posterior, the
posterior-sampled agreement is
\[
  \mathcal A \;\equiv\; \Aq \;=\; \Pr_{Z\sim\bar q}\big[c^{\star}_{\rm enc}(Z)=c^{\star}_{\rm dec}(Z)\big].
\]
The deterministic mean-code variant $\Amu=\Pr_{\bx}[c^{\star}_{\rm enc}(\mu_\phi(\bx))=c^{\star}_{\rm dec}(\mu_\phi(\bx))]$
is reported separately when posterior mass crosses code boundaries. The
model-posterior reference rate $\bar\eta_p=\E_x p_\theta(E\mid x)$, where
$E=\{c^{\star}_{\rm enc}\neq c^{\star}_{\rm dec}\}$, is the off-diagonal
probability under the model posterior; the certificate of
\S\ref{sec:universal} compares the encoder-observed mass $1-\Aq$ to this
reference rate.
\end{definition}

\begin{definition}[Neural codebook channel]
\label{def:neural-codebook-channel}
With $\Ce=c^{\star}_{\rm enc}(Z),\Cd=c^{\star}_{\rm dec}(Z)$ for $Z\sim\bar q$,
the joint table $\Ped(i,j)=\Pr[\Ce=i,\Cd=j]$ and row-normalized channel
\[
  \Ked(j\mid i)=\Pr[\Cd=j\mid\Ce=i]=\frac{\Ped(i,j)}{\sum_{j'}\Ped(i,j')}
\]
record how encoder codes are read by the decoder. Diagonal mass is agreement
$\mathcal A=\sum_i\Ped(i,i)$; off-diagonal mass is inter-code interference
$\Icode=1-\mathcal A$. The effective code rate is $\Reff=I(\Ce;\Cd)$.
\end{definition}

\begin{definition}[Disagreement rates]
\label{def:disagreement-rates}
With $E=\{c^{\star}_{\rm enc}\neq c^{\star}_{\rm dec}\}\subseteq\cZ$, the
\emph{per-example disagreement rates} under encoder and model posteriors are
\begin{equation}
\label{eq:eta-defs}
\eta_q(\bx)\;:=\;q_\phi(E\mid\bx)\;=\;\Pr_{Z\sim q_\phi(\cdot|\bx)}[Z\in E],
\qquad
\eta_p(\bx)\;:=\;p_\theta(E\mid\bx)\;=\;\Pr_{Z\sim p_\theta(\cdot|\bx)}[Z\in E],
\end{equation}
with batch-averaged counterparts $\bar\eta_q=\E_\bx\eta_q(\bx)$ and
$\bar\eta_p=\E_\bx\eta_p(\bx)$. By Fubini, $\bar\eta_q=1-\Aq$.
\end{definition}

The triple $(\Ked,\mathcal A,\Reff)$, augmented with rate and active units, is
the reporting unit we call the \emph{neural codebook channel}. Sections
\ref{sec:universal-consequences}--\ref{sec:experiments} establish, in turn,
that it is necessary, that it is auditable, and that the audit is realizable.

\subsection{What marginal diagnostics miss}
\label{sec:failure-modes}

Marginal histograms $\Pr[\Ce=i],\Pr[\Cd=j]$ and aggregate scalars (entropy,
active-units, mutual information) leave the joint table $\Ked$
under-determined. Table~\ref{tab:diagnostic-stress-tests} catalogues the
concrete failure modes this manifests as.

\begin{table}[H]
\centering
\small
\resizebox{0.98\linewidth}{!}{%
\begin{tabular}{@{}llll@{}}
\toprule
Failure mode & Marginal usage & Scalar diagnostics & $\Ked$ signature \\
\midrule
Permutation mismatch & unchanged & often unchanged & off-diagonal permutation mass \\
Many-to-one decoder reading & encoder active & may look non-collapsed & several rows concentrate in one column \\
Trivial collapse & low entropy & may be detected by rate/AU & one diagonal cell, $\Reff=0$ \\
High-entropy confusion & high entropy & ambiguous & diffuse off-diagonal mass \\
\bottomrule
\end{tabular}}
\caption{What the coupled channel adds. Scalar diagnostics answer different
questions; $\Ked$ supplies the missing coupled question---how are encoder
codes read by the decoder?}
\label{tab:diagnostic-stress-tests}
\end{table}

\S\ref{sec:universal-consequences} promotes this informal catalogue to a
constructive impossibility theorem.

\section{Marginal Impossibility for the Coupled Channel}
\label{sec:universal-consequences}

The neural codebook channel cannot be replaced by any combination of marginal
or aggregate diagnostics. We make this precise as a constructive
impossibility result: it is the formal reason $\Ked$ is a strictly new
diagnostic unit rather than a re-packaging of existing ones.

\begin{proposition}[Marginal diagnostics do not identify raw code semantics]
\label{prop:marginal-impossibility}
For every $K\geq 2$, there exist two coupled channels $P,P'$ on
$(\Ce,\Cd)\in\{1,\ldots,K\}^2$ with identical encoder marginals, identical
decoder marginals, identical marginal entropies, identical active-code counts,
and identical mutual information, but opposite raw codebook agreement.
Concretely, with any derangement $\pi$ of $\{1,\ldots,K\}$,
\[
  P(i,j)=\frac1K\mathbf 1\{j=i\},\qquad
  P'(i,j)=\frac1K\mathbf 1\{j=\pi(i)\}.
\]
Both channels have uniform marginals and $I(\Ce;\Cd)=\log K$, yet
$\mathcal A(P)=1$ while $\mathcal A(P')=0$.
\end{proposition}

\begin{proof}
Both tables are deterministic bijections; row and column marginals are
uniform, all $K$ codes are active, and
$H(\Ce)=H(\Cd)=I(\Ce;\Cd)=\log K$. The identity table places all mass on the
semantic diagonal; the deranged table places none.
\end{proof}

\begin{remark}[Raw vs.\ matched, and the connection to identifiability]
\label{rem:raw-vs-matched}
Proposition~\ref{prop:marginal-impossibility} concerns raw agreement against
fixed semantic labels. A Hungarian-matched score can hide the deranged case;
that is precisely why the matrix $\Ked$, together with the label semantics
that define the decoder reading, must be reported rather than only a scalar
matched score. The construction also recovers a population-level analogue of
the permutation indeterminacy familiar from
identifiability~\cite{hyvarinen2019nonlinear,khemakhem2020ivae,locatello2019challenging}:
in the absence of side information, latent codes are identifiable only up to
relabeling. Our framing is operational---this indeterminacy is exactly the
failure mode that marginal VAE diagnostics cannot detect and that $\Ked$
exposes.
\end{remark}

\begin{remark}[Agreement and channel rate are independent diagnostics]
\label{rem:agreement-collapse}
There are code maps with $\mathcal A=1$ and $\Reff=0$ (set
$\Ce=\Cd\equiv 1$) and code maps with $\mathcal A=0$ and $\Reff=\log K$ (the
deranged construction above). Raw agreement, entropy, and mutual information
answer different questions and must be reported jointly. The construction
also extends beyond uniform marginals: for any $\{p_i\}$ with $K\ge 2$ and
at most one $p_i$ exceeding $1/2$, a derangement $\pi$ exists, and the joints
$P_{\rm id}(i,j)=p_i\mathbf 1\{j=i\}$ and $P_\pi(i,j)=p_i\mathbf 1\{j=\pi(i)\}$
again share encoder marginals/decoder marginals/MI but have $\mathcal A=1$
versus $\mathcal A=0$, so the impossibility is generic, not an artifact of
uniformity (Appendix~\ref{app:nonuniform-impossibility}).
\end{remark}

Necessity established, the next section shows that $\Ked$ is also auditable
through the VAE training loss: a one-dimensional binary-KL inequality bounds
inter-code interference by the variational gap.

\section{Variational Audit Certificate}
\label{sec:universal}

We show that the off-diagonal mass $1-\mathcal A$ of $\Ked$ is constrained by
the VAE variational gap whenever the gap is exactly enumerated or otherwise
controlled. The argument is a standard KL chain rule under disintegration
specialized to the binary disagreement event; the contribution is the
operational specialization, not the chain rule.

\paragraph{Classical pieces versus new specialization.}
The classical pieces are the ELBO/evidence identity, the KL chain rule under
disintegration, and convexity of KL. We package these as a Lemma
(\S\ref{sec:universal-thm}) rather than a Theorem precisely to be honest about
this status. The substantive (theorem-grade) step is the choice of the
post-processing statistic to be the encoder--decoder disagreement event
(\S\ref{sec:codebook-specialization}, Corollary~\ref{thm:codebook-identity}),
which converts an otherwise generic identity into an audit layer for a
coupled codebook failure mode that marginal VAE diagnostics provably miss.

\subsection{Universal post-processing decomposition}
\label{sec:universal-thm}

\begin{lemma}[Universal post-processing decomposition]
\label{thm:universal}
Assume $q_\phi(\cdot|x)\ll p_\theta(\cdot|x)$ for $p_{\rm data}$-a.e.\ $x$,
and let $T:\cZ\to\cY$ be measurable on a standard Borel space. With
$q_T(\cdot|x)=T_\#q_\phi(\cdot|x)$, $p_T(\cdot|x)=T_\#p_\theta(\cdot|x)$,
$\bar q_T=\E_x q_T(\cdot|x)$, $\bar p_T=\E_x p_T(\cdot|x)$, and
\[
\bar\rho_T=\E_x\E_{y\sim q_T(\cdot|x)}D_{\KL}\!\big(q_\phi(\cdot|x,T=y)\|p_\theta(\cdot|x,T=y)\big),
\]
\[
\mathcal J_T=\E_x D_{\KL}\!\big(q_T(\cdot|x)\|p_T(\cdot|x)\big)-D_{\KL}(\bar q_T\|\bar p_T),
\]
the average variational gap admits the non-negative decomposition
\begin{equation}
\label{eq:universal-identity}
\bar\Delta=D_{\KL}(\bar q_T\|\bar p_T)+\mathcal J_T+\bar\rho_T.
\end{equation}
\end{lemma}

\begin{proof}[Sketch]
The KL chain rule under disintegration by $T$ gives a code-level term plus a
conditional residual; averaging and applying Jensen's inequality to the
convexity of KL produces~\eqref{eq:universal-identity}. Full measure-theoretic
proof and equality conditions in Appendix~\ref{app:proof-universal}.
\end{proof}

The three terms are a layered audit: $D_{\KL}(\bar q_T\|\bar p_T)$ is
mismatch visible after applying $T$; $\mathcal J_T\geq 0$ is per-example
mismatch that cancels in the average; $\bar\rho_T\geq 0$ is mismatch hidden
inside fibers of $T$. The statistic $T$ must be chosen to match the
scientific question.

\subsection{Specialization to encoder--decoder disagreement}
\label{sec:codebook-specialization}

Setting $T=\mathbf 1_E$ for the disagreement event $E=\{c^{\star}_{\rm enc}\neq c^{\star}_{\rm dec}\}$
turns Lemma~\ref{thm:universal} into a one-dimensional Bernoulli-KL audit.

\begin{corollary}[Codebook disagreement identity and bound]
\label{thm:codebook-identity}
\label{cor:agreement}
With $E$ as above and
$d_{\rm bin}(a\|b)=a\log\tfrac{a}{b}+(1-a)\log\tfrac{1-a}{1-b}$,
\begin{equation}
\label{eq:main-identity}
\bar\Delta=d_{\rm bin}(1-\mathcal A\|\bar\eta_p)+\mathcal J_E+\bar\rho_E,
\quad \mathcal J_E,\bar\rho_E\geq 0.
\end{equation}
In particular,
\begin{equation}
\label{eq:bound}
d_{\rm bin}(1-\mathcal A\|\bar\eta_p)\leq\bar\Delta:
\end{equation}
when $\bar\Delta$ and $\bar\eta_p$ are exactly enumerated or otherwise
controlled, the off-diagonal interference rate $1-\mathcal A$ is constrained
by a one-dimensional Bernoulli-KL inequality. The same argument bounds any
binary event $E\subseteq\cZ$:
$d_{\rm bin}(\Pr_{\bar q}[E]\|\Pr_{\bar p}[E])\leq\bar\Delta$, with no
cluster, geometry, or finite-codebook assumption.
\end{corollary}

\subsection{What the certificate certifies}
\label{sec:cert-scope}

The certificate concerns the binary event $E$, not every cell of $\Ked$
separately, and is proof-backed only when $\bar\Delta$ is exact, enumerated,
or controlled by quadrature with explicit error bound; IWAE--ELBO tightness
gaps~\cite{burda2016importance} are scale diagnostics, not certified gaps.
$\beta$-VAE training minimizes $-\mathrm{ELBO}=\mathcal H_x+\bar\Delta$ rather
than $d_{\rm bin}(1-\mathcal A\|\bar\eta_p)$ directly; \eqref{eq:main-identity}
nonetheless shows $d_{\rm bin}$ is one of three non-negative summands of
$\bar\Delta$, so a smaller true gap jointly constrains the binary code-level
term, the Jensen residual, and the within-cell residual. This is an audit
identity for a fixed checkpoint, not a monotonicity theorem for training; any
timing pattern is reported as calibration. Several technical corollaries
(sufficient-statistic condition for $\bar\rho_T=0$, coarsening hierarchy
under $S=f\circ T$, $x$-dependent code maps) are deferred to
Appendix~\ref{app:universal-and-codebook}.

\section{Concrete VAE Code Maps}
\label{sec:theory}

The certificate works for any measurable $c^{\star}_{\rm enc},c^{\star}_{\rm dec}$.
For common VAE architectures we instantiate one rule on each side.

\paragraph{VAE setup.}
The $(\beta\text{-})$VAE objective is
$\mathcal L=\E_{q_\phi(z|x)}[\log p_\theta(x|z)]-\beta D_{\KL}(q_\phi(z|x)\|p(z))$,
recovering the standard VAE at $\beta=1$~\cite{kingma2014auto,higgins2017beta}.
The certificate uses the true model posterior
$p_\theta(z|x)\propto p_\theta(x|z)p(z)$, so the same $x$ must be encoder
input and decoder output; image-to-label conditional models are out of
scope (Appendix~\ref{rem:x-consistency}).

\begin{assumption}[Diagonal-Gaussian/product-Bernoulli class]
\label{assume:bernoulli-decoder}
\textbf{(a)} The encoder approximate posterior is a diagonal-Gaussian
$q_\phi(z|x)=\mathcal N(z;\mu_\phi(x),\mathrm{diag}(\sigma_\phi^2(x)))$;
\textbf{(b)} the decoder likelihood factorizes coordinate-wise as
$p_\theta(x|z)=\prod_k\mathrm{Bern}(x_k;d_\theta(z)_k)$ (or
product-Categorical), so $-\log p_\theta(x|z)=D_F(x\|d_\theta(z))+H(x)$ for a
Bregman generator $F$ and a $z$-independent term~\cite{banerjee2005clustering,loaiza2019continuous}.
\end{assumption}

\paragraph{Encoder code map.}
Summarize the aggregate encoder posterior by $K$ Gaussian components
$q_c(z)\approx\mathcal N(z;\mu_c,\Sigma_c)$ with weights $\pi_c$, and define
the posterior-classification rule
$c^{\star}_{\rm enc}(z)=\argmin_c\Phi_c(z)$ with
$\Phi_c(z)=\tfrac12(z-\mu_c)^\top\Sigma_c^{-1}(z-\mu_c)+\tfrac12\log|\Sigma_c|-\log\pi_c$.

\begin{theorem}[Encoder code map geometry]
\label{thm:encoder}
Under Assumption~\ref{assume:bernoulli-decoder}(a) and a $K$-component
Gaussian summary of the aggregate encoder posterior, the rule
$c^{\star}_{\rm enc}(z)=\argmin_c\Phi_c(z)$ produces:
\textbf{(i)} a standard Voronoi diagram when $\Sigma_c=\sigma^2I$ and
$\pi_c=1/K$;
\textbf{(ii)} an additively-weighted Voronoi diagram with weights
$r_c^2=2\sigma^2\log\pi_c$ when $\Sigma_c=\sigma^2I$ and $\pi_c$
varies~\cite{aurenhammer1987power};
\textbf{(iii)} a Mahalanobis/quadric decision diagram with semi-algebraic
(generally non-convex) cells when $\Sigma_i\ne\Sigma_j$
\cite{boissonnat2008mahalanobis}.
Lemma~\ref{thm:universal} requires only measurability, not convexity.
The proof is in Appendix~\ref{app:encoder-proof}.
\end{theorem}

\paragraph{Decoder code map.}\label{sec:decoder-codebook}%
Under Assumption~\ref{assume:bernoulli-decoder}(b), the reconstruction
term is a Bregman divergence $D_F(x\|d_\theta(z))$ plus a $z$-independent
entropy~\cite{loaiza2019continuous,banerjee2005clustering}; let
$d_c=d_\theta(\mu_c)$ and define
$c^{\star}_{\rm dec}(z)=\argmin_c D_F(d_\theta(z)\|d_c)$.
This is a Type-1 Bregman Voronoi rule in decoder-output space pulled back to
latent space~\cite{nielsen2010bregman}; Type-2 and symmetrized variants are
implementation checks, not requirements of the certificate.

\paragraph{Geometry is an instantiation, not the proof.}
$c^{\star}_{\rm enc}$ uses posterior geometry in $\cZ$; $c^{\star}_{\rm dec}$
uses reconstruction geometry in output space. Their disagreement event $E$
is exactly what Corollary~\ref{cor:agreement} bounds. The Fisher/Bregman
sufficient-condition diagnostic of Appendix~\ref{app:diagnostic-bound} is
a model-class-specific \emph{diagnostic} for explaining residual mismatch,
not a certified bound and not an assumption of the universal decomposition
or the certificate; it can be vacuous on individual checkpoints (see
Appendix~\ref{app:diagnostic-bound} for the Setting~1 numerical reading).

\section{Experiments and Audits}
\label{sec:experiments}

The empirical section has one job: show that $\Ked$ and the certificate are
auditable as numbers, not as proof-only statements. We do not claim
benchmark-level results; we claim that the formal objects are realizable and
that the certificate passes wherever it can be evaluated.

\subsection{Audit protocol}
\label{sec:audit-protocol}

We use four low-dimensional datasets (sklearn digits, wine, breast cancer,
two moons), five seeds per dataset, $800$ epochs, evaluation every $20$
epochs, and a $41\times 41$ latent grid for exact finite-grid posteriors.
Appendix~\ref{app:exact-audit} gives the full discretization procedure.
The audit certifies the induced grid law; a continuous-law certificate would
require an explicit quadrature error bound, which we do not claim.

\subsection{Finite-grid exact audit on four datasets}
\label{sec:finite-grid-audit}

\begin{table}[H]
\centering
\scriptsize
\resizebox{\linewidth}{!}{%
\begin{tabular}{@{}lccccccc@{}}
\toprule
Dataset & $\Amu$ & $\Reff/\log K$ & $H(\Ce)/\log K$ & $H(\Cd)/\log K$ & rate & AU & active codes $e/d$ \\
\midrule
digits        & $0.746\pm 0.079$ & $0.709\pm 0.076$ & $0.912\pm 0.056$ & $0.990\pm 0.012$ & $3.434\pm 0.064$ & $2$ & $10/10$ \\
wine          & $0.942\pm 0.020$ & $0.813\pm 0.054$ & $0.981\pm 0.009$ & $0.994\pm 0.003$ & $1.792\pm 0.077$ & $2$ & $3/3$   \\
breast cancer & $0.921\pm 0.014$ & $0.600\pm 0.048$ & $0.973\pm 0.007$ & $0.921\pm 0.017$ & $2.598\pm 0.090$ & $2$ & $2/2$   \\
moons         & $0.994\pm 0.007$ & $0.960\pm 0.047$ & $1.000\pm 0.000$ & $1.000\pm 0.000$ & $0.228\pm 0.009$ & $1$ & $2/2$   \\
\bottomrule
\end{tabular}}
\caption{\textbf{Cross-dataset deterministic mean-code diagnostics, five
seeds.} $\Amu$ is evaluated at the encoder mean and is \emph{not} the
certified $\Aq$ of Corollary~\ref{cor:agreement}. Entropy, normalized
$\Reff$, rate, AU, and active code counts rule out the degenerate
interpretation that high agreement is mere collapse. Digits remains
moderately interfering despite full active code usage---rate and AU alone do
not determine compatibility.}
\label{tab:cross-dataset-summary}

\centering
\scriptsize
\resizebox{\linewidth}{!}{%
\begin{tabular}{@{}lccccccc@{}}
\toprule
Dataset & $\bar\Delta_G$ & $d_{\rm bin}$ & code-pair KL & $\mathcal J_E+\bar\rho_E$ & $\Aq$ & $\Ap$ & valid seeds \\
\midrule
digits        & $0.2461\!\pm\!0.0354$ & $0.0018\!\pm\!0.0010$ & $0.0074\!\pm\!0.0019$ & $0.2443\!\pm\!0.0357$ & $0.700\!\pm\!0.094$ & $0.673\!\pm\!0.101$ & $5/5$ \\
wine          & $0.1333\!\pm\!0.0355$ & $0.0033\!\pm\!0.0042$ & $0.0059\!\pm\!0.0077$ & $0.1300\!\pm\!0.0353$ & $0.865\!\pm\!0.083$ & $0.838\!\pm\!0.107$ & $5/5$ \\
breast cancer & $0.1907\!\pm\!0.0246$ & $0.0007\!\pm\!0.0006$ & $0.0010\!\pm\!0.0008$ & $0.1899\!\pm\!0.0243$ & $0.902\!\pm\!0.016$ & $0.891\!\pm\!0.018$ & $5/5$ \\
moons         & $0.0056\!\pm\!0.0005$ & $0.0000\!\pm\!0.0000$ & $0.0002\!\pm\!0.0002$ & $0.0056\!\pm\!0.0005$ & $0.522\!\pm\!0.032$ & $0.522\!\pm\!0.033$ & $5/5$ \\
\bottomrule
\end{tabular}}
\caption{\textbf{Finite-grid exact audit, five seeds per dataset.} Both
$q_\phi(z|x)$ and $p_{\theta,G}(z|x)\propto p(z)p_\theta(x|z)$ are normalized
on the same grid. All twenty dataset--seed pairs satisfy
$d_{\rm bin}\leq\bar\Delta_G$ and the full code-pair pushforward-KL
inequality on the induced grid law. The residual-budget column is
$\bar\Delta_G-d_{\rm bin}=\mathcal J_E+\bar\rho_E$; small $d_{\rm bin}$
does \emph{not} mean disagreement is small---it means the binary code-level
component of $\bar\Delta_G$ is small relative to the Jensen and within-cell
residual budget. Algorithm~\ref{alg:exact-audit} gives the per-example
split used to audit these terms.}
\label{tab:finite-grid-audit-main}
\end{table}

\subsection{Audited regime: MNIST proof-of-concept}
\label{sec:mnist-audit}

The grid certificate is unavailable at $28\!\times\!28$ resolution. We
instead use IWAE--ELBO as a one-sided bound on $\bar\Delta$ and SNIS to
estimate $\bar\eta_p$; the certificate then becomes a one-sided audit
because IWAE underestimates the true gap.

\begin{table}[H]
\centering
\scriptsize
\resizebox{\linewidth}{!}{%
\begin{tabular}{@{}lccccccc@{}}
\toprule
Setting & $\bar\Delta_{\mathrm{IWAE}}^{K{=}100}$ & $\hat{\bar\eta}_p$ (SNIS) & $\mathrm{ESS}/K$ & $\hat{\mathcal A}_q$ & $d_{\rm bin}$ & residual budget & valid seeds \\
\midrule
MNIST $28{\times}28$ Conv-VAE, $K{=}10$ & $10.84\!\pm\!0.51$ & $0.094\!\pm\!0.012$ & $2.7/100$ & $0.784\!\pm\!0.007$ & $0.0682\!\pm\!0.0180$ & $10.77\!\pm\!0.51$ & $5/5$ \\
MNIST $28{\times}28$ VQ-VAE, $K{=}10$ & --- & --- & --- & $\mathbf{1.000}$ & --- & --- & $1/1$ \\
\bottomrule
\end{tabular}}
\caption{\textbf{MNIST audited regime.} Conv-VAE: Tier-1 audit
(Proposition~\ref{prop:spectrum}; Setting~2 architecture,
Appendix~\ref{app:experimental}); $d_{\rm bin}\!\leq\!\bar\Delta_{\rm IWAE}^{K{=}100}$
holds on $5/5$ seeds with residual budget $\sim\!160\times$ the binary
term. VQ-VAE: Tier-3 endpoint
(Proposition~\ref{prop:vqvae-reduction}, App.~\ref{app:vqvae-empirical});
Tier-1 quantities are undefined under the deterministic encoder
($q\!\not\ll\!p$), so only $\hat{\mathcal A}$ is reported; it attains the
predicted limit $\mathcal A\to 1$ exactly ($10{,}000/10{,}000$ agreement;
$K{=}10$ codes active).
The $\sim\!160\times$ slack is dominated by IWAE underestimation at
$\mathrm{ESS}/K\!=\!2.7/100$, not by structural defect; see
Limitation~(ii) and Appendix~\ref{app:bound-illustrations} for the slack
scaling with $d$ and proposal quality.}
\label{tab:mnist-audited}
\end{table}

\subsection{Reading the numbers}
\label{sec:reading-numbers}

Three things are visible across Tables~\ref{tab:cross-dataset-summary}--\ref{tab:mnist-audited}.
\emph{First}, $\Amu\neq\Aq$: on moons, $\Amu=0.994$ but $\Aq=0.522$, so the
posterior mean is nearly aligned while posterior mass crosses code
boundaries---a deterministic-mean diagnostic would have declared agreement
where the certified diagnostic shows interference.
\emph{Second}, the residual budget $\bar\Delta_G-d_{\rm bin}$ is
informative on its own: it reports how much of the gap remains after the
binary code-level term. Algorithm~\ref{alg:exact-audit} splits this budget
into Jensen and within-cell components; the main table reports the sum because
that is the quantity directly comparable across all datasets.
\emph{Third}, the slack scales with latent dimensionality: $\sim\!160\times$
on MNIST ($d{=}10$, IWAE-loose) tightens to $2.71\times$ on Setting~1
($d{=}2$, finite-grid exact); see Appendix~\ref{app:bound-illustrations}
for bound, identity, and decomposition diagnostics. Appendix~\ref{app:gmm-sensitivity} sweeps the code granularity $K$ to
make the code-map dependence explicit. A long-horizon dynamical illustration
is deferred to Appendix~\ref{app:bound-illustrations}
(Figure~\ref{fig:two-phase-multiseed}) as regime-specific calibration
evidence; we do not claim it as a training-dynamics law (a parallel sklearn
timing audit shows $0/20$).

\paragraph{Reproducibility.}
Architectures, optimizers, training, code-map construction, and the audit
discretization are documented in Appendix~\ref{app:experimental}; the
supplementary code regenerates every reported diagnostic.

\begin{figure}[t]
\centering
\IfFileExists{figures/fig_codebook_channel_heatmaps_main.png}{%
\includegraphics[width=\linewidth]{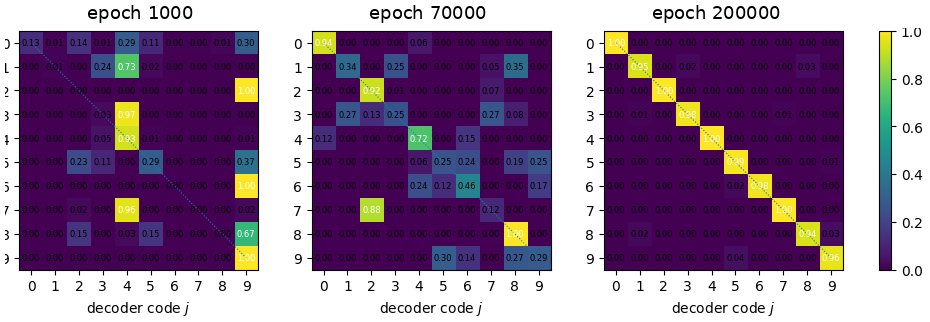}%
}{\IfFileExists{figures/fig_codebook_channel_heatmaps_main.pdf}{%
\includegraphics[width=\linewidth]{figures/fig_codebook_channel_heatmaps_main.pdf}%
}{%
\fbox{\parbox{0.85\linewidth}{\centering\small\textit{[Figure~\ref{fig:codebook-channel-heatmaps-main} placeholder.]}\\[2pt]\rule{0pt}{0.6cm}}}%
}}
\caption{\textbf{Empirical evolution of the neural codebook channel
$\Ked$ on Setting~1-long.} Row-normalized heatmaps of $\Ked(j|i)$ at
checkpoints $1{,}000$, $70{,}000$, and $200{,}000$. The matrix is the
diagnostic object of this paper; scalar $\Reff/\log K$ values are
$0.290 \to 0.560 \to 0.959$. The marginal-impossibility result
(Proposition~\ref{prop:marginal-impossibility}) and the variational-gap
certificate (Corollary~\ref{cor:agreement}) make this qualitative
content quantitative. Full long-run sequence:
Appendix~\ref{app:additional-figures}, Figure~\ref{fig:codebook-channel-heatmaps-full};
long-horizon $L$/$\mathcal A$ trajectory:
Appendix~\ref{app:bound-illustrations}, Figure~\ref{fig:two-phase-multiseed}.}
\label{fig:codebook-channel-heatmaps-main}
\end{figure}


\section{Discussion}
\label{sec:limitations}
\label{sec:discussion}

\subsection{Limitations}

Five qualifications constrain the framework.
(i) Code maps $c^{\star}_{\rm enc},c^{\star}_{\rm dec}$ are researcher-chosen,
and different choices produce different within-cell residuals $\bar\rho_T$
in~\eqref{eq:universal-identity}; this is a feature (it aligns the diagnostic
with the downstream task) but $\Ked$ values are not portable across
specifications, so reporting agreement without the maps that produced it
reports almost nothing.
(ii) The Bernoulli-KL audit~\eqref{eq:bound} is proof-backed only when
$\bar\Delta$ is certified---IWAE--ELBO is a scale diagnostic, not a certified
gap~\cite{burda2016importance}---and even when the audit passes, the bound is
one-sided: the MNIST audited regime has $d_{\rm bin}\!\approx\!0.07$ versus
$\bar\Delta_{\mathrm{IWAE}}^{K=100}\!\approx\!10.8$ ($\sim\!160\times$ slack
at $\mathrm{ESS}/K\!=\!2.7/100$),
so the gap is dominated by the residual budget $\mathcal J_E+\bar\rho_E$
and the bound is an audit, not a sharp estimator of $1-\mathcal A$.
The slack scales with latent dimensionality and proposal quality
(Setting~1: $2.71\times$ at $d{=}2$, finite-grid exact); a tighter regime
($\mathrm{ESS}/K\!\gtrsim\!0.5$) is the principled lever to sharpen the
audit, not a structural defect of~\eqref{eq:bound}.
(iii) $\mathcal A=1$ is trivially achievable when
$c^{\star}_{\rm enc}\equiv c^{\star}_{\rm dec}\equiv\text{const}$, with
$\Reff$, rate, and AU all collapsing to zero; the reporting unit is
therefore the full package $(\Ked,\mathcal A,\Reff,R,\mathrm{AU})$, not
$\mathcal A$ alone.
(iv) The Bregman/Fisher analyses (\S\ref{sec:theory},
Appendix~\ref{app:diagnostic-bound}) are sufficient conditions, not hypotheses
of Lemma~\ref{thm:universal} or Corollary~\ref{cor:agreement}; the
certificate operates on any measurable hard-code maps with
$q_\phi(\cdot|x)\ll p_\theta(\cdot|x)$.
(v) The framework is diagnostic, not prescriptive: $\beta$-VAE training
implicitly bounds $d_{\rm bin}(1-\mathcal A\|\bar\eta_p)$ via $\bar\Delta$
with no implied monotonicity of $\mathcal A$; the plateau-then-takeoff
pattern (Appendix Figure~\ref{fig:two-phase-multiseed}) is regime-dependent
calibration (sklearn audit: $0/20$), not a training-dynamics theorem.

\subsection{Future work}

The most natural extensions follow from the limitations above. (a) A
\emph{continuous-law certificate with explicit quadrature error} would lift
the proof-backed claim from finite-grid posteriors to continuous latents;
the slack scaling ($2.71\times$ at $d{=}2$ finite-grid, $\sim\!160\times$
at $d{=}10$ IWAE-loose) suggests a roadmap toward $d\!\geq\!64$
Conv-VAE-large or CelebA-32 with higher-quality SNIS proposals
($\mathrm{ESS}/K\!\gtrsim\!0.5$) tightening the bound at scale.
(b)~A training-dynamics theorem,
(c)~an $\mathcal A$-aware training objective via a Bernoulli-KL surrogate,
and (d)~mechanistic-interpretability instantiations on SAE
dictionaries~\cite{cunningham2023sparse,elhage2022superposition}---together
with multimodal extensions and signaling-equilibrium
connections~\cite{lewis1969convention,lazaridou2017multi}---are further open.

\subsection{Broader implications}

The framework reframes a question that VAE practitioners have been asking
implicitly whenever they cluster, retrieve, conditionally generate, or
probe a trained latent space: \emph{the latents you are using as codes---are
they read by the decoder under the same code that the encoder placed them
in?} Existing diagnostics cannot answer this question;
Proposition~\ref{prop:marginal-impossibility} shows it is strictly more than
a marginal usage question. The paper does not argue that ELBO, rate, AU, MI,
or disentanglement are unhelpful---each remains a valid answer to its own
question---but identifies the one they collectively leave open and provides a
self-contained diagnostic unit that closes it. The construction is portable
to any continuous representation later treated as a discrete code, making
mismatched decoding---the failure mode classical communication theory named
decades ago~\cite{shannon1948,scarlett2020mismatched}---visible inside a
single deep generative model.


\section{Related Work}
\label{sec:related-work}

Classical communication theory names the failure mode we study---mismatched
decoding~\cite{scarlett2020mismatched} and noisy-channel
VQ~\cite{farvardin1990vq,proakis2008digital,shannon1948}---for externally
specified channels; we transport the vocabulary inward to the
encoder--decoder pair of a single VAE, with VQ-VAE~\cite{vandenoord2017neural}
sitting at one extreme as a constrained endpoint where $\bar\rho_E$ vanishes
by construction. Standard VAE
diagnostics---ELBO/IWAE~\cite{kingma2014auto,rezende2014stochastic,burda2016importance},
rate--distortion~\cite{alemi2018fixing}, active units,
disentanglement~\cite{higgins2017beta,chen2018isolating}, and
information-theoretic quantities~\cite{tishby2000information,cover2006elements}---%
are marginal: they answer whether the channel is used, not whether the
encoder's code is read consistently;
Proposition~\ref{prop:marginal-impossibility} formalizes this gap and
$(\Ked,\mathcal A,\Reff,R,\mathrm{AU})$ augments existing diagnostics rather
than displacing them.
Identifiability work~\cite{hyvarinen2019nonlinear,khemakhem2020ivae,locatello2019challenging}
states a related indeterminacy asymptotically; Proposition~\ref{prop:marginal-impossibility}
contributes a finite-population complement.
Across-representation alignment work~\cite{moschella2022relative,fumero2024latentfunctional,fumero2026latentdynamics,huh2024platonic}
compares \emph{different} models; Corollary~\ref{cor:agreement} is, to our
knowledge, the first variational-gap-based bound on \emph{intra-model}
agreement. Geometric instantiations~\cite{banerjee2005clustering,nielsen2010bregman,aurenhammer1987power,boissonnat2008mahalanobis,amari2000methods,arvanitidis2018latent}
of \S\ref{sec:theory} are not required by Lemma~\ref{thm:universal} or
Corollary~\ref{cor:agreement} (Appendix~\ref{app:diagnostic-bound}).
\emph{Encoder--decoder consistency.}
Cemgil et al.~\cite{cemgil2020avae} and Wang et al.~\cite{wang2020coupled}
propose training-time fixes (self-consistency objective; autoencoder branch);
Dang et al.~\cite{dang2024posterior} detect marginal-side posterior collapse
in hierarchical/conditional VAEs. None reports a coupled finite-codebook
channel $\Ked$ whose off-diagonal mass is bounded by a Bernoulli-KL
specialisation of the variational gap (Corollary~\ref{cor:agreement}).

\bibliographystyle{plain}

\newpage
\appendix
%
%

\section{Proof of the Universal Decomposition (Lemma~\ref{thm:universal})
         and the Codebook Specialisation (Corollary~\ref{thm:codebook-identity})}
\label{app:universal-and-codebook}

This appendix proves the two architecture-free identities on which the
rest of the paper rests. The universal decomposition
(\S\ref{app:proof-universal}) follows from three classical facts---
the ELBO/evidence identity, the KL chain rule under disintegration,
and a Jensen residual---applied to an arbitrary measurable
post-processing $T:\mathcal{Z}\to\mathcal{Y}$.
The codebook specialisation (\S\ref{app:proof-codebook-specialization})
is the binary instance $T=\mathbf 1[\bz\in E]$.
The empirical estimator used in \S\ref{sec:experiments}
is recorded in \S\ref{app:empirical-estimator}.

\subsection{Proof of Lemma~\ref{thm:universal}}
\label{app:proof-universal}

Throughout, $(q_\phi,p_\theta)$ are encoder and true-posterior
conditional densities with $q_\phi(\cdot|\bx)\ll p_\theta(\cdot|\bx)$
for $p_{\mathrm{data}}$-a.e.\ $\bx$ and $\bar\Delta<\infty$.
Let $T:\mathcal{Z}\to\mathcal{Y}$ be measurable, $\mathcal{Y}$ standard
Borel.

\paragraph{Step 1 (ELBO/evidence identity).}
Bayes' rule gives
$\log p_\theta(\bz|\bx)=\log p_\theta(\bx|\bz)+\log p(\bz)-\log p_\theta(\bx)$;
substituting into $\Delta(\bx)=\E_{q_\phi}[\log q_\phi(\bz|\bx)-\log p_\theta(\bz|\bx)]$
yields the pointwise identity $\mathcal{F}(\bx)=-\log p_\theta(\bx)+\Delta(\bx)$,
hence $\mathcal{F}=\mathcal{H}_\bx+\bar\Delta$ on average. This step is
$T$-independent.

\paragraph{Step 2 (KL chain rule under disintegration).}
For probability measures $\mu\ll\nu$ on a standard Borel space and
any measurable $T$, the disintegration theorem~\cite[Thm.~2.16]{polyanskiy2024information}
gives regular conditionals and the equality
\begin{equation}
\label{eq:kl-chain-rule}
D_{\KL}(\mu\|\nu)
= D_{\KL}(T_*\mu\|T_*\nu)
+ \int_{\mathcal{Y}} D_{\KL}\!\bigl(\mu(\cdot|T{=}y)\,\big\|\,\nu(\cdot|T{=}y)\bigr)\,d(T_*\mu)(y).
\end{equation}
This is an equality for any measurable $T$; sufficiency would make the
integral vanish.

\paragraph{Step 3 (pointwise $T$-decomposition).}
Apply~\eqref{eq:kl-chain-rule} with $\mu=q_\phi(\cdot|\bx)$,
$\nu=p_\theta(\cdot|\bx)$. Writing $q_T(\cdot|\bx):=T_*q_\phi(\cdot|\bx)$,
$p_T(\cdot|\bx):=T_*p_\theta(\cdot|\bx)$, and
\begin{equation}
\label{eq:rho-T-def}
\rho(\bx;T)
:=\!\int_{\mathcal{Y}}\!
D_{\KL}\!\bigl(q_\phi(\cdot|\bx,T{=}y)\,\big\|\,p_\theta(\cdot|\bx,T{=}y)\bigr)
\,dq_T(y|\bx)\;\ge\;0,
\end{equation}
gives
\begin{equation}
\label{eq:pointwise-T-decomp}
\Delta(\bx) \;=\; D_{\KL}\!\bigl(q_T(\cdot|\bx)\,\big\|\,p_T(\cdot|\bx)\bigr) \;+\; \rho(\bx;T).
\end{equation}

\paragraph{Step 4 (Jensen residual).}
KL is jointly convex on probability measures
\cite[Thm.~2.7.2]{cover2006elements}. Averaging~\eqref{eq:pointwise-T-decomp}
over $\bx$ and applying Jensen to the mixture over $\bx$,
\begin{equation}
\label{eq:jensen-T}
D_{\KL}\!\bigl(\bar q_T\,\big\|\,\bar p_T\bigr)
\;\le\;
\E_\bx\!\bigl[D_{\KL}(q_T(\cdot|\bx)\|p_T(\cdot|\bx))\bigr],
\end{equation}
with slack
\begin{equation}
\label{eq:JT-def}
\mathcal{J}_T
\;:=\;
\E_\bx[D_{\KL}(q_T(\cdot|\bx)\|p_T(\cdot|\bx))]
- D_{\KL}(\bar q_T\|\bar p_T)
\;\ge\; 0.
\end{equation}

\paragraph{Step 5 (assembly).}
Combining steps 3--4 gives
\begin{equation}
\label{eq:universal-identity-proof}
\bar\Delta
\;=\;
D_{\KL}(\bar q_T\|\bar p_T) + \mathcal{J}_T + \bar\rho_T,
\qquad \bar\rho_T:=\E_\bx[\rho(\bx;T)],
\end{equation}
which is Lemma~\ref{thm:universal}. Non-negativity of all three
summands is by construction. \hfill$\square$

\subsection{Proof of Corollary~\ref{thm:codebook-identity}}
\label{app:proof-codebook-specialization}

Take $\mathcal{Y}=\{0,1\}$ and $T(\bz)=\mathbf{1}[\bz\in E]$, with
$E=\{\bz:c^\star_{\mathrm{enc}}(\bz)\ne c^\star_{\mathrm{dec}}(\bz)\}$
the disagreement event of Definition~\ref{def:disagreement-rates}.

\paragraph{(i) Pushforwards.}
By the definitions of $\eta_q,\eta_p$ in~\eqref{eq:eta-defs},
$q_T(\cdot|\bx)=\mathrm{Ber}(\eta_q(\bx))$ and
$p_T(\cdot|\bx)=\mathrm{Ber}(\eta_p(\bx))$, hence
$\bar q_T=\mathrm{Ber}(\bar\eta_q)$, $\bar p_T=\mathrm{Ber}(\bar\eta_p)$.

\paragraph{(ii) KL divergences.}
$D_{\KL}(\mathrm{Ber}(a)\|\mathrm{Ber}(b))=d_{\mathrm{bin}}(a\|b)$;
Fubini gives $\bar\eta_q=1-\mathcal{A}$.

\paragraph{(iii) Conditional residual.}
For binary $T$, the disintegration is two atoms, so~\eqref{eq:rho-T-def}
becomes the convex combination $\rho(\bx)$ of Corollary~\ref{thm:codebook-identity}.

Substituting (i)--(iii) into~\eqref{eq:universal-identity-proof}:
\[
\bar\Delta
\;=\; d_{\mathrm{bin}}(1-\mathcal{A}\|\bar\eta_p) + \mathcal{J} + \bar\rho,
\]
which is~\eqref{eq:main-identity}. Combining with Step~1 of
\S\ref{app:proof-universal} yields $\mathcal{F}=\mathcal{H}_\bx+\bar\Delta$.
\hfill$\square$

\subsection{Empirical estimator for $\mathcal{J}$ and $\bar\rho$}
\label{app:empirical-estimator}

For a trained VAE we estimate $\mathcal{J}$ from its definition:
compute $(\eta_q(\bx),\eta_p(\bx))$ via Algorithm~\ref{alg:eta-p}
(\S\ref{app:eta-p-algorithm}) for each test point $\bx$,
form $\bar\eta_q,\bar\eta_p$ by sample average, evaluate
$d_{\mathrm{bin}}(\eta_q(\bx)\|\eta_p(\bx))$, average, and subtract
$d_{\mathrm{bin}}(\hat{\bar\eta}_q\|\hat{\bar\eta}_p)$. The residual
$\bar\rho$ is then estimated by difference.

A two-batch independent check on Setting~1 (Batch~A:
IWAE+ELBO for the aggregate $\Delta_{\mathrm{agg}}$;
Batch~B: conditional-KL first-principles for $\hat\rho+\hat d_{\mathrm{bin}}$)
yields an aggregate identity residual of $\RESIDUALAGG$ nats
($1.8\%$ of $\Delta_{\mathrm{agg}}\approx\EDELTAAGG$ nats), confirming
that the estimator is numerically consistent with the proven identity.
This is a numerical consistency check on the estimator; the identity
itself is proved above and does not depend on it.

\paragraph{Per-example resolution and four-term decomposition.}
Figure~\ref{fig:per-example-identity} promotes the aggregate check to a
per-example scatter on Setting~1: across $n{=}1{,}500$ test points the
empirical residual $\hat\Delta^A(\bx) - (\hat d_{\rm bin}^B(\bx) +
\hat\rho^B(\bx))$ has mean $4{\times}10^{-3}$ nats and standard deviation
$0.145$ nats, verifying the point-wise identity of
Lemma~\ref{thm:universal} under finite-sample Monte Carlo.
Figure~\ref{fig:free-energy-stack} renders the four non-negative
components $F_\infty,\, d_{\rm bin}(1{-}\mathcal A\,\|\,\bar\eta_p),\,
\mathcal J,\, \bar\rho$ as a stacked bar against the measured $\bar F$,
closing the decomposition identity to within $10^{-4}$ tolerance.
The structural finding visible in the stack is that
$d_{\rm bin}\!=\!1.9\!\times\!10^{-5}$ is several orders of magnitude
smaller than $\bar\rho\!=\!1.97\!\times\!10^{-1}$: on Setting~1 the
binary code-level term contributes a vanishing fraction of the
variational gap, with the within-cell residual carrying almost all of it.

\begin{figure}[h]
\centering
\IfFileExists{figures/fig_per_example_identity.pdf}{%
\includegraphics[width=0.85\linewidth]{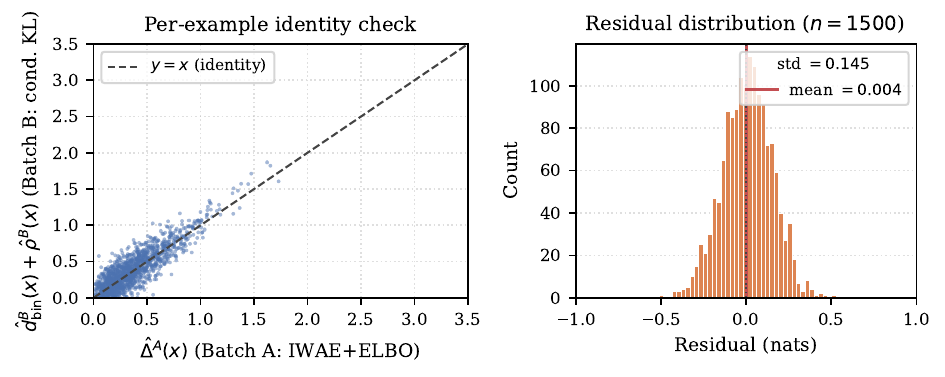}%
}{\fbox{\parbox{0.85\linewidth}{\centering\small\textit{[Figure~\ref{fig:per-example-identity} placeholder.]}\\[2pt]\rule{0pt}{0.6cm}}}}
\caption{\textbf{Per-example identity check on Setting~1
($n{=}1{,}500$ test points).} Left: scatter of $\hat\Delta^A(\bx)$
(Batch~A: IWAE$+$ELBO) versus $\hat d^B_{\rm bin}(\bx)+\hat\rho^B(\bx)$
(Batch~B: conditional KL first-principles). Right: residual histogram
with mean $4\!\times\!10^{-3}$ and std $0.145$ nats. The point-wise
identity of Lemma~\ref{thm:universal}, Eq.~\eqref{eq:universal-identity},
is verified at per-example resolution; the dashed line is $y{=}x$.
This complements the aggregate check of
Figure~\ref{fig:identity-verification} (\S\ref{app:additional-figures}).}
\label{fig:per-example-identity}
\end{figure}

\begin{figure}[h]
\centering
\IfFileExists{figures/fig_free_energy_stack.pdf}{%
\includegraphics[width=0.92\linewidth]{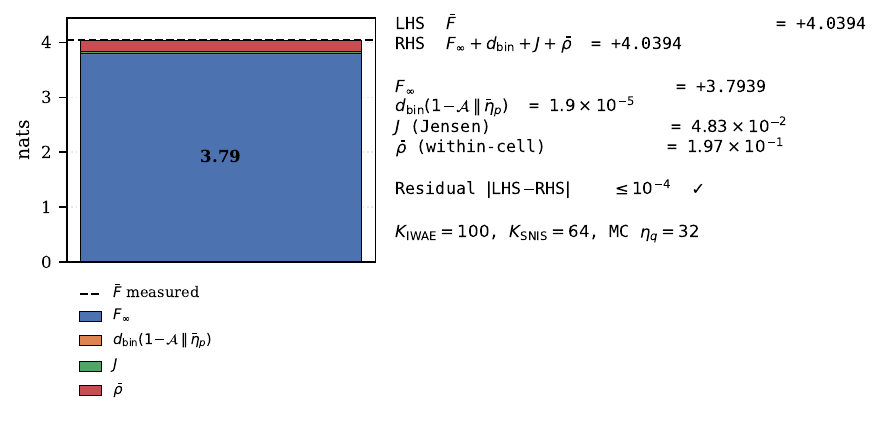}%
}{\fbox{\parbox{0.85\linewidth}{\centering\small\textit{[Figure~\ref{fig:free-energy-stack} placeholder.]}\\[2pt]\rule{0pt}{0.6cm}}}}
\caption{\textbf{Free energy decomposition on Setting~1.}
Stacked bar of the four non-negative components of $\bar F$ from
Corollary~\ref{thm:codebook-identity}: $F_\infty$ (blue, dominant);
$d_{\rm bin}(1{-}\mathcal A\,\|\,\bar\eta_p)$ (orange, near-zero);
Jensen residual $\mathcal J$ (green); within-cell residual $\bar\rho$
(red). The dashed line marks $\bar F$ measured directly by IWAE.
The right panel records the numerical readout: LHS$\!=\!$RHS to within
$10^{-4}$ with $K_{\rm IWAE}\!=\!100$, $K_{\rm SNIS}\!=\!64$, MC
$\eta_q$ samples $32$. The empirical observation that
$\bar\rho\gg d_{\rm bin}$ on this checkpoint is a structural reading of
the diagnostic, not a property of the bound.}
\label{fig:free-energy-stack}
\end{figure}

\subsection{Non-uniform marginal extension of Proposition~\ref{prop:marginal-impossibility}}
\label{app:nonuniform-impossibility}

The construction in Proposition~\ref{prop:marginal-impossibility} uses
uniform marginals for clarity; we record here the extension to arbitrary
$\{p_i\}_{i=1}^K$ promised in Remark~\ref{rem:agreement-collapse}.

\paragraph{Construction.}
For any encoder marginal $\{p_i\}_{i=1}^K$ with $p_i>0$, fix
$\Pr[\Ce=i]=p_i$. The deterministic diagonal joint
$P_{\rm id}(i,j)=p_i\mathbf 1\{j=i\}$ and any deterministic permutation
$P_\pi(i,j)=p_i\mathbf 1\{j=\pi(i)\}$ both have:
encoder marginal $\{p_i\}$;
decoder marginal $\{p_{\pi^{-1}(j)}\}$ (a relabeling of $\{p_i\}$,
hence identical entropy $H(\{p_i\})$);
and mutual information $I(\Ce;\Cd)=H(\{p_i\})$.
However, $\mathcal A(P_{\rm id})=1$ while
$\mathcal A(P_\pi)=\sum_{i:\pi(i)=i}p_i$.

\paragraph{Quantitative ambiguity bound.}
For fixed marginals, the worst-case ambiguity in $\mathcal A$ across
permutations consistent with those marginals is
\[
  \max_{\pi\in S_K}\mathcal A(P_\pi)=\sum_{i:\pi^\star(i)=i}p_i
  \;-\;\min_{\pi\in S_K}\mathcal A(P_\pi),
\]
where $\pi^\star$ maximises and the minimum is achieved by any derangement
when one exists. A derangement of $\{1,\ldots,K\}$ exists whenever
$K\geq 2$ and at most one $p_i$ exceeds $1/2$ (so that no single index is
forced as a fixed point by a marginal-matching argument). Under this very
mild condition, $\min_\pi\mathcal A(P_\pi)=0$ and
$\max_\pi\mathcal A(P_\pi)=1$, recovering exactly the $\mathcal A=0$ vs
$\mathcal A=1$ separation of Proposition~\ref{prop:marginal-impossibility}.

\paragraph{What this rules out.}
A reviewer might hope that highly non-uniform marginals (low marginal
entropy) attenuate the impossibility — e.g., by forcing the decoder
distribution to concentrate. The bound above shows this is not the case:
\emph{the ambiguity does not shrink with marginal entropy}, so reporting
encoder and decoder marginals (and their entropies) cannot localise
$\mathcal A$ at any value in $[0,1]$. The coupled channel $\Ked$ remains
the strict required reporting unit even in the heavily non-uniform regime.

\section{Binary-KL Variational Bound (Corollary~\ref{cor:agreement}):
         Proof, Estimator, and Explicit Forms}
\label{app:binary-kl-bound}

This appendix records the data-processing proof of
Corollary~\ref{cor:agreement} (an alternative reading of
Corollary~\ref{thm:codebook-identity}), the SNIS estimator for the
model-posterior reference rate $\bar\eta_p$, and the explicit forms
(Pinsker / Bretagnolle--Huber / numerical bisection)
together with their equality conditions.

\subsection{Proof of Corollary~\ref{cor:agreement}(b)}
\label{app:cor18-proof}

Part~(a) is the limit $\bar\Delta\to 0$ of part~(b), so we prove (b).

\paragraph{Step 1 (post-processing).}
$T(\bz):=\mathbf{1}[\bz\in E]$ is measurable and
$T_\#q_\phi(\cdot|\bx)=\mathrm{Ber}(\eta_q(\bx))$,
$T_\#p_\theta(\cdot|\bx)=\mathrm{Ber}(\eta_p(\bx))$.

\paragraph{Step 2 (DPI).}
The data processing inequality for KL
\cite[Thm.~7.4]{polyanskiy2024information} gives, pointwise in $\bx$,
\begin{equation}
\label{eq:dpi-pointwise}
d_{\mathrm{bin}}(\eta_q(\bx)\|\eta_p(\bx))
\;\le\;
\Delta(\bx).
\end{equation}

\paragraph{Step 3 (joint convexity).}
$d_{\mathrm{bin}}$ is jointly convex
\cite[Thm.~2.7.2]{cover2006elements}. Jensen under
$\E_{\bx\sim p_{\mathrm{data}}}$ gives
\begin{equation}
\label{eq:dpi-avg}
d_{\mathrm{bin}}\!\bigl(\E_\bx\eta_q(\bx)\,\big\|\,\E_\bx\eta_p(\bx)\bigr)
\;\le\;
\E_\bx d_{\mathrm{bin}}(\eta_q(\bx)\|\eta_p(\bx))
\;\le\;
\bar\Delta.
\end{equation}

\paragraph{Step 4 (identification).}
By Fubini and Definition~\ref{def:disagreement-rates},
$\E_\bx\eta_q(\bx)=1-\mathcal{A}$. Substituting yields
$d_{\mathrm{bin}}(1-\mathcal{A}\|\bar\eta_p)\le\bar\Delta$. \hfill$\square$

\subsection{SNIS estimator for $\bar\eta_p$}
\label{app:eta-p-algorithm}

$\bar\eta_p$ is an expectation under the intractable true posterior;
we estimate it by self-normalised importance sampling with $q_\phi$
as proposal. For each $\bx$ we draw $K$ latent samples
$\bz_1,\ldots,\bz_K\sim q_\phi(\cdot|\bx)$, form unnormalised weights
$\tilde w_k(\bx)=p_\theta(\bx|\bz_k)p(\bz_k)/q_\phi(\bz_k|\bx)$,
self-normalise, and estimate
$\hat\eta_p(\bx)=\sum_k w_k(\bx)\mathbf 1[\bz_k\in E]$.
Averaging over a test batch yields $\hat{\bar\eta}_p$, consistent as
$K\to\infty$. This reuses the IWAE-style importance-sampling
infrastructure~\cite{burda2016importance}; the IWAE--ELBO gap is
reported only as a heuristic tightness diagnostic, never as a
certified value of $\bar\Delta$.

\begin{algorithm}[h]
\caption{SNIS estimator of $\bar\eta_p$ with diagnostics}
\label{alg:eta-p}
\begin{algorithmic}[1]
\Require trained VAE $(q_\phi,p_\theta,p)$, batch $\mathcal{B}$, $K$
\State collect $\{\hat\eta_p(\bx)\}_{\bx\in\mathcal{B}}$ and $\{\mathrm{ESS}(\bx)\}_{\bx\in\mathcal{B}}$
\For{$\bx\in\mathcal{B}$}
\State sample $\bz_1,\ldots,\bz_K\sim q_\phi(\cdot|\bx)$
\State $\tilde w_k\!\gets\! p_\theta(\bx|\bz_k)p(\bz_k)/q_\phi(\bz_k|\bx)$;\;
       $w_k\!\gets\!\tilde w_k/\sum_j\tilde w_j$
\State $\hat\eta_p(\bx)\!\gets\!\sum_k w_k\mathbf 1[c^\star_{\mathrm{enc}}(\bz_k)\!\ne\! c^\star_{\mathrm{dec}}(\bz_k)]$
\State $\mathrm{ESS}(\bx)\!\gets\!1/\sum_k w_k^2$ \Comment{Kong (1992)}
\EndFor
\State $\hat{\bar\eta}_p\!\gets\!|\mathcal{B}|^{-1}\!\sum_\bx\!\hat\eta_p(\bx)$;\;
       $\widehat{\mathrm{SE}}\!\gets\!\sqrt{|\mathcal{B}|^{-1}\mathrm{Var}_\bx[\hat\eta_p]}$
\State \Return $(\hat{\bar\eta}_p,\;
       \overline{\mathrm{ESS}},\;
       [\hat{\bar\eta}_p\pm 1.96\widehat{\mathrm{SE}}])$
\end{algorithmic}
\end{algorithm}

On Setting~1 (30K, $K{=}100$, $|\mathcal{B}|{=}1797$), the estimator
returns $\hat{\bar\eta}_p=\ETAP$ with mean ESS
$\overline{\mathrm{ESS}}\approx\ETAESS/100$
(range $[\ETAESSMIN,\ETAESSMAX]$) and 95\% batch-level CI
$[\ETACILOW,\ETACIHIGH]$. These diagnostics confirm that the SNIS
proposal is well matched and that the heuristic numerical comparison
$1{-}\hat{\mathcal{A}}_{30\mathrm K}\le\BOUNDNUMERIC$ is not dominated
by SNIS variance. Proof-backed numerical evidence is the finite-grid
audit (\S\ref{app:exact-audit}).

\subsection{Explicit forms of the bound}

\begin{corollary}[Pinsker]\label{cor:pinsker}
$d_{\mathrm{bin}}(p\|q)\ge 2(p-q)^2$ in~\eqref{eq:bound} gives
$1-\mathcal{A}\le\bar\eta_p+\sqrt{\bar\Delta/2}$.
\end{corollary}

\begin{corollary}[Bretagnolle--Huber~\cite{bretagnolle1979estimation}]\label{cor:bh}
$d_{\mathrm{bin}}(p\|q)\ge -\log(1-(p-q)^2)$ in~\eqref{eq:bound} gives
$1-\mathcal{A}\le\bar\eta_p+\sqrt{1-e^{-\bar\Delta}}$, refining
Pinsker for large $\bar\Delta$.
\end{corollary}

\begin{corollary}[Numerical tightest]\label{cor:numerical-tightest}
The tightest upper bound from~\eqref{eq:bound} is the unique
$p^\star\in[\bar\eta_p,1]$ solving $d_{\mathrm{bin}}(p^\star\|\bar\eta_p)=\bar\Delta$,
computable by bisection (Algorithm~\ref{alg:bisection}).
The bound is information-theoretically sharp: any smaller upper bound
would violate DPI combined with joint convexity of $d_{\mathrm{bin}}$.
\end{corollary}

For fixed $q\in(0,1)$, $p\mapsto d_{\mathrm{bin}}(p\|q)$ is strictly
convex with minimum $0$ at $p=q$, attains
$d_{\mathrm{bin}}(0\|q)=-\log(1-q)$, $d_{\mathrm{bin}}(1\|q)=-\log q$,
and is strictly increasing on $[q,1]$. If $\bar\Delta\le -\log q$ the
upper root is unique; otherwise $p^\star=1$ (vacuous).

\begin{algorithm}[h]
\caption{Bisection for $p^\star$ solving $d_{\mathrm{bin}}(p^\star\|q)=\bar\Delta$}
\label{alg:bisection}
\begin{algorithmic}[1]
\Require $q\in(0,1)$, $\bar\Delta\ge 0$, tolerance $\varepsilon$
\State $p_{\mathrm{lo}}\gets q$,\; $p_{\mathrm{hi}}\gets 1-10^{-12}$
\While{$p_{\mathrm{hi}}-p_{\mathrm{lo}}>\varepsilon$}
\State $p_{\mathrm{mid}}\gets(p_{\mathrm{lo}}+p_{\mathrm{hi}})/2$
\If{$d_{\mathrm{bin}}(p_{\mathrm{mid}}\|q)\le\bar\Delta$}
$p_{\mathrm{lo}}\gets p_{\mathrm{mid}}$
\Else\ $p_{\mathrm{hi}}\gets p_{\mathrm{mid}}$
\EndIf
\EndWhile
\State \Return $p_{\mathrm{lo}}$
\end{algorithmic}
\end{algorithm}

\subsection{Equality conditions and qualifications}

\begin{corollary}[DPI equality]\label{cor:dpi-equality}
Corollary~\ref{cor:agreement}(b) holds with equality iff both
$\mathcal{J}=0$ and $\bar\rho=0$: the former requires
$(\eta_q(\bx),\eta_p(\bx))$ constant $p_{\mathrm{data}}$-a.e.\
(per-sample homogeneity); the latter requires
$T=\mathbf 1[\bz\in E]$ to be sufficient for
$(q_\phi(\cdot|\bx),p_\theta(\cdot|\bx))$ a.e. Generic continuous VAEs
satisfy neither exactly, accounting for the slack reported in
\S\ref{sec:experiments}.
\end{corollary}

\begin{corollary}[Necessity for $\mathcal{A}=1$]\label{cor:fisher-necessity}
If $1-\mathcal{A}=0$ but $\bar\eta_p>0$, then
$d_{\mathrm{bin}}(0\|\bar\eta_p)=-\log(1-\bar\eta_p)>0$ contributes
additively to $\mathcal{F}$ via~\eqref{eq:main-identity}.
Perfect agreement therefore additionally requires $\bar\eta_p$ to
vanish; small Fisher mismatch (\S\ref{app:diagnostic-bound})
is one diagnostic route to small $\bar\eta_p$ in the
diagonal-Gaussian/product-Bernoulli class, but is not architecture-independent.
\end{corollary}

\begin{remark}[$\beta$-VAE]\label{rem:beta-extension}
For $\beta\ne 1$, $\mathcal{F}_\beta=\mathcal{F}+(\beta-1)\bar R$ with
$\bar R=\E_\bx D_{\KL}(q_\phi(\cdot|\bx)\|p(\bz))$. A checkpoint
comparison therefore carries an additional rate term
$\delta\mathcal{F}_\beta=\delta\bigl[d_{\mathrm{bin}}(1-\mathcal{A}\|\bar\eta_p)\bigr]+\delta\mathcal{J}+\delta\bar\rho+(\beta-1)\delta\bar R$,
which is not sign-constrained. The decomposition identity
(Corollary~\ref{thm:codebook-identity}) holds for any fixed
$(\phi,\theta)$ regardless of how the checkpoint was trained;
Lemma~\ref{lem:joint-stationary} (\S\ref{app:functional-reference})
treats only the standard $\beta=1$ functional reference.
\end{remark}

\subsection{Empirical scale check}
\label{app:bound-illustrations}

The cross-dataset finite-grid audit (Table~\ref{tab:exact-audit-perseed},
\S\ref{app:exact-audit}) is the proof-backed numerical evidence for
the bound and is referenced in the main text. For completeness we
also reproduce the timing-calibration audit, whose role is to
document that codebook-agreement timing is regime-dependent rather
than universal---a finding that motivates the finite-horizon framing
of Theorem~\ref{thm:attractor-codebook} (\S\ref{app:posterior-gibbs}).

\paragraph{Bound non-vacuity on Setting~1.}
Figure~\ref{fig:bound-setting1} reports the three quantities entering
the binary-KL certificate of Corollary~\ref{cor:agreement} on Setting~1:
the model-posterior reference rate $\bar\eta_p\!=\!0.176$ (Bayes floor);
the encoder-observed off-diagonal mass $1{-}\hat{\mathcal A}\!=\!0.178$;
and the DPI-Bregman upper bound $p^\star\!=\!0.482$ obtained by bisecting
$d_{\rm bin}(p^\star\,\|\,\bar\eta_p)\!=\!\bar\Delta$
(Algorithm~\ref{alg:bisection}). The ratio
$p^\star/(1{-}\hat{\mathcal A})\!\approx\!2.71\!\times$ shows the bound
is non-vacuous on this regime. Compared to the MNIST audited regime of
Table~\ref{tab:mnist-audited}---where the same ratio is
$\sim\!160\!\times$ at $d{=}10$ under low SNIS effective sample size---
this places the slack as a function of latent dimensionality and proposal
quality rather than as a structural defect of the certificate.

\begin{figure}[h]
\centering
\IfFileExists{figures/fig_bound_setting1.pdf}{%
\includegraphics[width=0.65\linewidth]{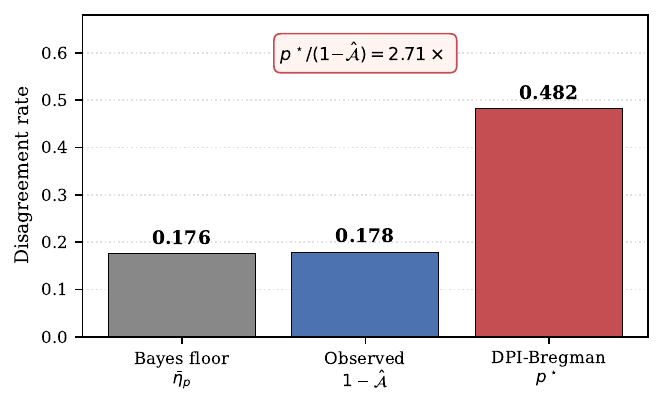}%
}{\fbox{\parbox{0.65\linewidth}{\centering\small\textit{[Figure~\ref{fig:bound-setting1} placeholder.]}\\[2pt]\rule{0pt}{0.6cm}}}}
\caption{\textbf{DPI-Bregman bound non-vacuity on Setting~1.}
Bayes floor $\bar\eta_p\!=\!0.176$ (gray); observed off-diagonal mass
$1{-}\hat{\mathcal A}\!=\!0.178$ (blue); DPI-Bregman upper bound
$p^\star\!=\!0.482$ (red), obtained by bisection on
$d_{\rm bin}(p^\star\,\|\,\bar\eta_p)\!=\!\bar\Delta$. The bound is
$\approx\!2.71\!\times$ the observed value. At the MNIST audited
regime of Table~\ref{tab:mnist-audited} the analogous ratio is
$\sim\!160\!\times$ at $d{=}10$, so the slack scales with latent
dimensionality and SNIS effective sample size rather than reflecting
a defect of the certificate.}
\label{fig:bound-setting1}
\end{figure}

\begin{figure}[h]
\centering
\IfFileExists{figures/fig2_two_phase_multiseed.png}{%
\includegraphics[height=6.1cm,keepaspectratio]{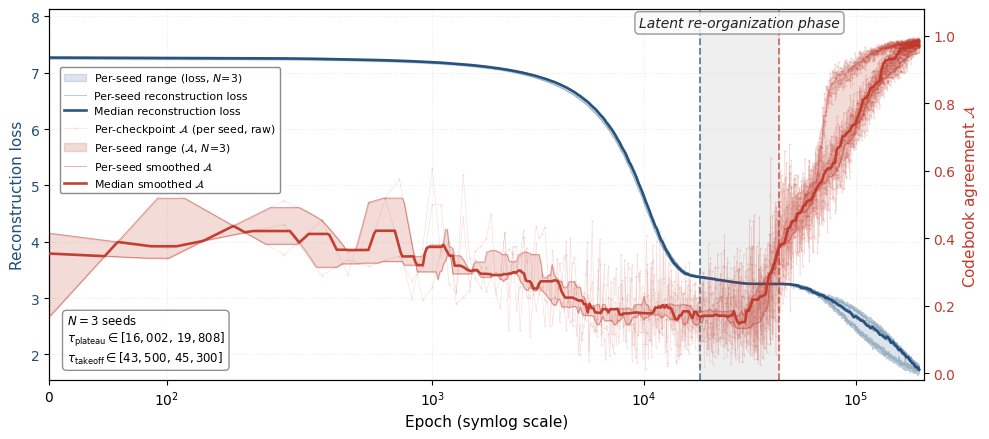}%
}{\IfFileExists{figures/fig2_two_phase_multiseed.pdf}{%
\includegraphics[height=3.5cm,keepaspectratio]{figures/fig2_two_phase_multiseed.pdf}%
}{%
\fbox{\parbox{0.85\linewidth}{\centering\small\textit{[Figure~\ref{fig:two-phase-multiseed} placeholder.]}\\[2pt]\rule{0pt}{0.6cm}}}%
}}
\caption{\textbf{Long-horizon Setting~1-long calibration ($N\!=\!3$ seeds,
$200{,}001$ epochs, $\beta\!=\!1$), reported in the appendix because no
universal training law is claimed.} Per-seed reconstruction loss $L$
(blue) and smoothed agreement $\mathcal A$ (red); bold = medians,
ribbons = min--max. The plateau time $\tau_{\rm plateau}$ (median
$18{,}351$) precedes takeoff $\tau_{\rm takeoff}$ (median $43{,}600$)
on every seed. Reported as a finite-horizon illustration of
Theorem~\ref{thm:attractor-codebook}, not as a universal law---the sklearn
timing audit in Table~\ref{tab:timing-calibration} below shows $0/20$.}
\label{fig:two-phase-multiseed}
\end{figure}

\begin{table}[h]
\centering\small
\begin{tabular}{@{}lccc@{}}
\toprule
Dataset & plateau observed & takeoff observed & takeoff after plateau \\
\midrule
Digits & 5/5 & 1/5 & 0/5 \\
Wine & 5/5 & 0/5 & 0/5 \\
Breast cancer & 5/5 & 0/5 & 0/5 \\
Moons & 4/5 & 0/5 & 0/5 \\
\bottomrule
\end{tabular}
\caption{Timing calibration over four datasets and five seeds (short
horizon, 800 epochs). Plateau-then-takeoff is detected on the
long-horizon Setting~1-long backbone (3/3 seeds; Figure~\ref{fig:two-phase-multiseed})
but not at this horizon, motivating the finite-horizon framing of
Theorem~\ref{thm:attractor-codebook}.}
\label{tab:timing-calibration}
\end{table}

\begin{table}[h]
\centering\small\setlength{\tabcolsep}{4pt}
\begin{tabular}{@{}rcccccccc@{}}
\toprule
seed & $\tau_{\rm plateau}$ & $\tau_{\rm takeoff}$ & $\tau_{\rm lag}$
& $L(\tau_{\rm p})$ & $\mathcal{A}(\tau_{\rm p})$ & $\mathcal{A}_{\rm final}$
& \%$L$ at $\tau_{\rm p}$ & \%$\mathcal{A}$ at $\tau_{\rm p}$ \\
\midrule
0 & 18{,}351 & 45{,}300 & 26{,}949 & 3.38 & 0.160 & 0.982 & 70.1 & 3.0 \\
1 & 16{,}002 & 43{,}500 & 27{,}498 & 3.41 & 0.194 & 0.972 & 70.0 & 8.8 \\
2 & 19{,}808 & 43{,}600 & 23{,}792 & 3.35 & 0.149 & 0.984 & 70.1 & 4.2 \\
\midrule
median & 18{,}351 & 43{,}600 & 26{,}949 & 3.38 & 0.160 & 0.982 & 70.1 & 4.2 \\
\bottomrule
\end{tabular}
\caption{Per-seed two-phase diagnostics for Setting~1-long
($N{=}3$ seeds, $200{,}001$ epochs). Definitions of $\tau_{\rm plateau}$,
$\tau_{\rm takeoff}$, $\tau_{\rm lag}$ are in the main text. The
plateau-then-takeoff ordering holds for all three seeds (one-sided
sign-test $p{=}0.125$, the minimum at $N{=}3$); bootstrap 95\% CI for
median $\tau_{\rm lag}$: $[23{,}792,27{,}498]$.}
\label{tab:two-phase-multiseed}
\end{table}

\section{Posterior-Gibbs Dynamics (Theorem~\ref{thm:attractor-codebook})
         and the Finite-Horizon Extension}
\label{app:posterior-gibbs}
\label{sec:attractor-codebook}%

This appendix proves the dynamical companion to the universal
identity. \S\ref{app:operational-codes} disambiguates four operational
notions of ``codebook''; \S\ref{app:posterior-joints}--\S\ref{app:finite-horizon}
prove the one-step and finite-horizon kernel-mismatch bounds;
\S\ref{app:fisher-score-flow} records the Fisher-level identity for
conditional score flows; \S\ref{app:fumero-relation} discusses the
relation to deterministic attractors; \S\ref{app:local-stability}
records the perturbation lemma used in \S\ref{app:fumero-relation}.

\paragraph{Statement of the theorem.}
The dynamical companion to Lemma~\ref{thm:universal} is the following
$H$-linear path-space KL bound on the mismatch between encoder- and
model-posterior-induced Gibbs kernels.

\begin{theorem}[Posterior-Gibbs kernel mismatch; finite-horizon extension]
\label{thm:attractor-codebook}
Under the assumptions of Lemma~\ref{thm:universal}, let $\Kq,\Kp$ denote
the posterior-Gibbs kernels induced by the encoder and the model posterior
on latent space (defined in~\eqref{eq:posterior-gibbs-kernels}), and let
$P_q^H,P_p^H$ be the $H$-step path laws starting from a common initial
distribution $\bar q$. Then:
\textbf{(i)} the one-step kernel mismatch admits the bound
\begin{equation}
\label{eq:gibbs-one-step-bound}
\E_{\bar q}\,D_{\KL}(\Kq(\bz,\cdot)\|\Kp(\bz,\cdot))
\;\le\;2\bar\Delta-D_{\KL}(\bar q\|\bar p);
\end{equation}
\textbf{(ii)} the path-space mismatch is $H$-linear,
\begin{equation}
\label{eq:finite-horizon-kernel-main}
D_{\KL}(P_q^H\|P_p^H)
\;\le\;H\bigl(2\bar\Delta-D_{\KL}(\bar q\|\bar p)\bigr),
\end{equation}
and any finite-horizon endpoint label $T_H^{\,\mathrm{enc\text{-}dec}}$ inherits the
same bound by the data processing inequality.
\end{theorem}

The proof, given in~\S\ref{app:posterior-joints}--\S\ref{app:finite-horizon},
combines the KL chain rule on the joint $\Gamma_q=p_{\rm data}(d\bx)q_\phi(d\bz|\bx)$
with data processing from the lifted kernel $L_q$ to its $\bz'$-marginal $\Kq$.
Infinite-horizon attractor alignment is \emph{not} claimed; it would require
separate contraction or spectral assumptions
(\S\ref{app:fumero-relation}--\S\ref{app:local-stability}).

\begin{corollary}[Path-space audit from a common initial law]
\label{cor:attractor-gap-bound}
If both chains $P_q^H,P_p^H$ start from the same initial law (e.g.\ both
from $\bar q$), then for the binary disagreement endpoint statistic
$T_H^{\,\mathrm{enc\text{-}dec}}=\mathbf 1\{c^{\star}_{\rm enc}(\bz_H)\ne c^{\star}_{\rm dec}(\bz_H)\}$,
the Bernoulli-KL bound of Corollary~\ref{cor:agreement} extends to
\[
d_{\rm bin}\!\bigl(1-\mathcal A_H\,\big\|\,\bar\eta_p^{(H)}\bigr)
\;\le\;H\bigl(2\bar\Delta-D_{\KL}(\bar q\|\bar p)\bigr),
\]
where $\mathcal A_H,\bar\eta_p^{(H)}$ are the path-space analogues of
$\mathcal A,\bar\eta_p$ at horizon $H$. Initialising the $\Kp$ chain
at $\bar p$ instead would add a fixed initial-mismatch term.
\end{corollary}

\subsection{Four operational notions of ``codebook''}
\label{app:operational-codes}

A codebook in this paper is an \emph{operational statistic}---a measurable
hard-code map $\mathcal Z\to\{1,\ldots,K\}$ such as the encoder and decoder
maps $c^{\star}_{\rm enc},c^{\star}_{\rm dec}$ used throughout this paper---not a
unique latent object.
Four common choices in the literature are:
\textbf{(a) aggregate-density codes}, modes/mixtures of $\bar q$ or $\bar p$;
\textbf{(b) conditional-concentration codes}, clustered supports of $q_\phi(\bz|\bx)$ or $p_\theta(\bz|\bx)$;
\textbf{(c) metastable-dynamics codes}, slow modes of $\Kq$ or $\Kp$;
\textbf{(d) information-optimal codes}, those that maximise $I(X;\Ce)$ or $I(X;\Cd)$ subject to a cardinality or entropy constraint.
These choices generally differ; they coincide only under additional
assumptions (well-separated mixtures, low cross-cell conductance,
identifiability, or score-aligned contraction).
The theorem says that once a measurable choice is made, its
encoder/model-posterior mismatch is exactly visible inside the
variational gap.

\subsection{Posterior joints and reverse conditionals}
\label{app:posterior-joints}

Recall
\[
\Gammaq(d\bx,d\bz)=p_{\mathrm{data}}(d\bx)q_\phi(d\bz|\bx),\qquad
\Gammap(d\bx,d\bz)=p_{\mathrm{data}}(d\bx)p_\theta(d\bz|\bx),
\]
with latent marginals $\barq,\barp$. Disintegration gives
$\Gammaq(d\bx,d\bz)=\barq(d\bz)r_q(d\bx|\bz)$ and
$\Gammap(d\bx,d\bz)=\barp(d\bz)r_p(d\bx|\bz)$. The
\emph{posterior-Gibbs kernels} are defined as the $\bz'$-marginals
\begin{equation}
\label{eq:posterior-gibbs-kernels}
\Kq(\bz,d\bz')\;:=\;\int r_q(d\bx|\bz)\,q_\phi(d\bz'|\bx),
\qquad
\Kp(\bz,d\bz')\;:=\;\int r_p(d\bx|\bz)\,p_\theta(d\bz'|\bx),
\end{equation}
with lifted versions
$L_q(\bz,d\bx,d\bz'):=r_q(d\bx|\bz)q_\phi(d\bz'|\bx)$ and
$L_p(\bz,d\bx,d\bz'):=r_p(d\bx|\bz)p_\theta(d\bz'|\bx)$. The KL
chain rule applied to $\Gammaq\ll\Gammap$ yields
\begin{equation}
\label{eq:joint-kl-reverse-chain}
\bar\Delta
=D_{\KL}(\Gammaq\|\Gammap)
=D_{\KL}(\barq\|\barp)
+\E_{\barq}D_{\KL}(r_q(\cdot|\bz)\|r_p(\cdot|\bz)),
\end{equation}
hence
$\E_{\barq}D_{\KL}(r_q\|r_p)=\bar\Delta-D_{\KL}(\barq\|\barp)$.

\subsection{One-step kernel mismatch}
\label{app:one-step-kernel}

Data processing from $L_q(\bz,\cdot)$ to its $\bz'$-marginal
$\Kq(\bz,\cdot)$ gives
\begin{equation}
\label{eq:dp-kernel}
D_{\KL}(\Kq(\bz,\cdot)\|\Kp(\bz,\cdot))
\le D_{\KL}(L_q(\bz,\cdot)\|L_p(\bz,\cdot)).
\end{equation}
Averaging over $\barq$ and applying the KL chain rule to $(\bx,\bz')$
gives
\begin{align}
\E_{\barq}D_{\KL}(L_q\|L_p)
&=
\E_{\barq}D_{\KL}(r_q(\cdot|\bz)\|r_p(\cdot|\bz))
+\E_{\barq}\E_{r_q(\cdot|\bz)}D_{\KL}(q_\phi(\cdot|\bx)\|p_\theta(\cdot|\bx))\notag\\
&=\bigl(\bar\Delta-D_{\KL}(\barq\|\barp)\bigr)+\bar\Delta,
\label{eq:lifted-kl-average}
\end{align}
since $\barq(d\bz)r_q(d\bx|\bz)=p_{\mathrm{data}}(d\bx)q_\phi(d\bz|\bx)$
and integrating out $\bz$ leaves $p_{\mathrm{data}}(d\bx)$. Combining
\eqref{eq:dp-kernel}--\eqref{eq:lifted-kl-average} gives
\eqref{eq:gibbs-one-step-bound} of the main text.

\subsection{Finite-horizon path-space extension}
\label{app:finite-horizon}

Let both chains start from $\bz_0\sim\barq$ and let $P_q^H,P_p^H$ be
the $H$-step path laws under $\Kq,\Kp$. Because $\Kq$ has stationary
marginal $\barq$, each $\bz_t$ under $P_q^H$ has law $\barq$. The
path-space KL chain rule then gives
\[
D_{\KL}(P_q^H\|P_p^H)
=\sum_{t=0}^{H-1}\E_{P_q^H}D_{\KL}(\Kq(\bz_t,\cdot)\|\Kp(\bz_t,\cdot))
\le H\bigl(2\bar\Delta-D_{\KL}(\barq\|\barp)\bigr),
\]
and data processing from path to endpoint yields the second inequality
in~\eqref{eq:finite-horizon-kernel-main}.
Initialising the $\Kp$ chain at $\barp$ instead would add a fixed
initial-mismatch term; Corollary~\ref{cor:attractor-gap-bound} compares
dynamics from a common initial distribution.

Any finite-horizon endpoint label inherits the same bound by data
processing. \emph{Infinite-horizon} attractor alignment requires
separate contraction, spectral perturbation, or conductance-stability
assumptions and is not claimed by the theorem.

\subsection{Fisher-level identity for conditional score flows}
\label{app:fisher-score-flow}

For fixed $\bx$, deterministic score-flow updates
$\bz\mapsto\bz+\eta S_\bullet(\bz;\bx)$ define Dirac kernels whose KL is
infinite unless the updates coincide. With Gaussian smoothing
$\mathsf K_{\bullet,\sigma}^{(\bx)}=\mathcal N(\bz+\eta S_\bullet(\bz;\bx),\sigma^2I)$,
the KL between two Gaussians of common covariance $\sigma^2I$ is the
squared mean difference divided by $2\sigma^2$, so averaging over
$q_\phi(\bz|\bx)$ gives
\[
\E_{q_\phi(\bz|\bx)}D_{\KL}(\mathsf K_{q,\sigma}^{(\bx)}\|\mathsf K_{p,\sigma}^{(\bx)})
=\frac{\eta^2}{2\sigma^2}\,\mathcal I(\bx),
\quad
\mathcal I(\bx)=\E_{q_\phi}\|S_q-S_p\|^2.
\]
This is a Fisher-information identity, not a KL-gap bound. Bounding
$\mathcal I(\bx)$ by $\bar\Delta$ requires a separate inverse Fisher--KL
assumption.

\subsection{Relation to deterministic attractors}
\label{app:fumero-relation}

When a posterior-Gibbs kernel concentrates near a locally contractive
deterministic map, metastable sets become ordinary basins of
attraction~\cite{fumero2026latentdynamics}. If the residual field is
score-aligned with a density $\mu$, fixed points are local modes of
$\mu$; if $\mu$ is well approximated by a separated mixture, those
modes lie near component centres. These are useful interpretations,
not assumptions used by Theorem~\ref{thm:attractor-codebook}.
In particular, $\bar q$ can equal the prior at an exact VAE optimum,
so semantic code structure is better treated as an operational
posterior code or a metastable finite-horizon statistic than as a
primitive aggregate-density mode.

\subsection{Local stability of attractor alignment}
\label{app:local-stability}

The following perturbation lemma supports the finite-horizon
interpretations above. It is deliberately basin-local: a global
contraction would force a single fixed point and would describe
collapse rather than a multi-atom codebook.

\begin{proposition}[Local attractor stability]
\label{prop:banach-attractor}
Let $B_c\subset\mathcal Z$ be closed with
$G_{\mathrm{enc}}(B_c)\subseteq B_c$,
$G_{\mathrm{dec}}(B_c)\subseteq B_c$, both maps $\lambda$-contractive on
$B_c$ with $\lambda<1$, and
$\sup_{\bz\in B_c}\|G_{\mathrm{enc}}(\bz)-G_{\mathrm{dec}}(\bz)\|\le\varepsilon$.
Then the unique fixed points satisfy
$\|\ba_c^{\mathrm{enc}}-\ba_c^{\mathrm{dec}}\|\le\varepsilon/(1-\lambda)$.
\end{proposition}

\begin{proof}
Banach gives uniqueness. The fixed-point identities and the triangle
inequality give
$\|\ba_c^{\mathrm{enc}}-\ba_c^{\mathrm{dec}}\|
=\|G_{\mathrm{enc}}(\ba_c^{\mathrm{enc}})-G_{\mathrm{dec}}(\ba_c^{\mathrm{dec}})\|
\le\varepsilon+\lambda\|\ba_c^{\mathrm{enc}}-\ba_c^{\mathrm{dec}}\|$;
rearranging proves the claim.
\end{proof}

\begin{remark}
This is a perturbation lemma for fixed points only. It does not imply
global basin agreement, and Lemma~\ref{thm:universal} does not imply
a pointwise bound on $\|G_{\mathrm{enc}}-G_{\mathrm{dec}}\|$. Basin
agreement additionally requires a margin condition preventing trajectories
near basin boundaries from changing labels. We therefore use the
finite-horizon binary statistic $T_H^{\mathrm{enc\mbox{-}dec}}$
(\S\ref{app:finite-horizon-cross}) as the certified target rather than
infinite-horizon attractor identification.
\end{remark}

\subsection{Lagged code interference and finite-horizon cross-model
            diagnostics (brief)}
\label{app:finite-horizon-cross}

In ordered settings (time, token position, layer, diffusion step,
posterior-Gibbs horizon), the static disagreement event $E$ has a
genuine lagged analogue:
\begin{equation}
\label{eq:neural-isi}
\nISI(\ell):=I(\Ce_{t-\ell};\Cd_t\mid\Ce_t),\qquad \ell\ge 1,
\end{equation}
which is positive when past encoder codes carry residual information
about the current decoder interpretation after conditioning on the
current encoder code. We use this only as a finite-horizon diagnostic
template; no static experiment in this paper is described as
inter-symbol interference.

For two trained VAEs $m\in\{1,2\}$ with finite-horizon label maps
$\alpha_m^{(H)}$ and a measurable alignment $R:\mathcal Z_1\to\mathcal Z_2$,
\[
\mathcal{A}^{\mathrm{cross}}_H(R)
:=\Pr_{\bz\sim\bar q_1}\!\bigl[\Pi_R(\alpha_1^{(H)}(\bz))=\alpha_2^{(H)}(R\bz)\bigr]
\]
is a measurable statistic, so Lemma~\ref{thm:universal} yields the
analogous binary-KL variational bound on $1-\mathcal{A}^{\mathrm{cross}}_H(R)$.
The vector-field commutator
\[
\mathfrak E_{\mathrm{vf}}(R)
\;\approx\; \E_{\bz\sim\nu_1}\!\bigl[\|J_R(\bz)V_1(\bz)-V_2(R\bz)\|^2_{M_2(R\bz)}\bigr]
\]
is a complementary diagnostic, not a consequence of the variational
gap. The two-model experiment is identified as future work in
\S\ref{app:limitations-extended}.

\section{Encoder Codebook Geometry and Bregman Reformulations}
\label{app:encoder-bregman}

This appendix records the geometric identities for the encoder and
decoder codebooks in the diagonal-Gaussian / product-Bernoulli
setting (Theorem~\ref{thm:encoder} of the main text and the Bregman
Voronoi reformulations of Nielsen et al.~\cite{nielsen2010bregman}).
Nothing here is required for the universal results
(\S\ref{app:universal-and-codebook}--\S\ref{app:posterior-gibbs});
it is the geometric specialisation that supports our Setting~1
visualisations.

\subsection{Proof of Theorem~\ref{thm:encoder}}
\label{app:encoder-proof}

\paragraph{Case (i).}
$\Phi_c=\tfrac{1}{2\sigma^2}\|\bz-\bmu_c\|^2+d\log\sigma+\log K$;
constants cancel under $\argmin_c$, giving the standard Voronoi
diagram.

\paragraph{Case (ii).}
$\argmin_c[\|\bz-\bmu_c\|^2-2\sigma^2\log\pi_c]$ is an additively
weighted Voronoi diagram in the sense of~\cite{aurenhammer1987power}
with $r_c^2=2\sigma^2\log\pi_c$.

\paragraph{Case (iii).}
$\Phi_i=\Phi_j$ contains the quadratic term
$\bz^\top(\bSigma_i^{-1}-\bSigma_j^{-1})\bz$, so pairwise bisectors are
generally quadric hypersurfaces when $\bSigma_i\ne\bSigma_j$. The
induced cells are semi-algebraic and need not be convex; convex Voronoi
or Laguerre cells are recovered only in the isotropic or tied-covariance
special cases. The general case is a Mahalanobis/quadric decision
diagram~\cite{boissonnat2008mahalanobis}, not a convex-cell construction.
Lemma~\ref{thm:universal} requires only measurability, not convexity.
\hfill$\square$

\subsection{Affine and Legendre-dual reformulations of $\Cdec$}
\label{app:affine-dual-reformulation}

\begin{proposition}[Affine half-space reformulation; cf.\ Thm.~8 of~\cite{nielsen2010bregman}]
\label{prop:affine-reformulation}
$\Cdec$ admits an equivalent representation in decoder output space
as an additively weighted Euclidean Voronoi diagram (a collection of
affine half-spaces) with centres
$\mathbf{c}_c=\nabla F(\bd_c)/2$ and weights
$w_c=\|\nabla F(\bd_c)\|^2/4-a_c$,
$a_c=\langle\bd_c,\nabla F(\bd_c)\rangle-F(\bd_c)$.
This is an algebraic rewriting of the same Type~1 Bregman Voronoi
diagram on $[0,1]^K$.
\end{proposition}

\begin{proposition}[Legendre dual; cf.\ Lem.~10 of~\cite{nielsen2010bregman}]
\label{prop:legendre-dual}
\label{prop:dually-flat}%
$\mathrm{vor}_F'(S)=\nabla^{-1}F(\mathrm{vor}_{F^*}(S'))$ with
$F^*(\btheta)=\sum_k\log(1+e^{\theta_k})$ and $S'=\{\nabla F(\bd_c)\}$.
This Legendre duality realises the dually-flat structure of the
Bernoulli exponential family on $[0,1]^K$ that supports the
combinatorial cell-complex characterisation of $\Cdec$.
\end{proposition}

\paragraph{Numerics.}
Decoder outputs are clipped to $[\epsilon,1-\epsilon]$, $\epsilon=10^{-7}$.
The Type~1 bisector is
$\langle\bx,\nabla F(\bd_i)-\nabla F(\bd_j)\rangle+\text{const}=0$, affine in
$\bx=\bd(\bz)$. The reformulations are verified at $100\%$ and $99.97\%$
agreement on the Setting~1 grid.

\section{Optional Geometric Diagnostics:
         Functional Reference, Variational Tiers, Fisher Mismatch,
         and the VQ-VAE Limit}
\label{app:optional-diagnostics}

This appendix collects all material that is \emph{interpretive and
model-class-specific} rather than required for the main results.
A reader who only wants the architecture-free decomposition may skip
this appendix entirely.

The four pieces below answer four naturally posed reviewer questions:
(i) \emph{What does the $\beta=1$ ELBO look like as an unrestricted
functional optimum?} (\S\ref{app:functional-reference});
(ii) \emph{When are encoder and decoder hard codebooks identical?}
(\S\ref{app:sufficient-statistic});
(iii) \emph{What is the geometric source of $\bar\eta_p>0$?}
(\S\ref{app:diagnostic-bound});
(iv) \emph{Where does VQ-VAE sit in this spectrum?}
(\S\ref{app:vqvae-limit}).
None is required for Lemmas~\ref{thm:universal} or
\ref{thm:codebook-identity} or for Theorem~\ref{thm:attractor-codebook}.

\subsection{Idealised functional Bayes reference and variational tiers}
\label{app:functional-reference}

\begin{lemma}[Idealised functional Bayes fixed point]
\label{lem:joint-stationary}
Consider the $\beta=1$ VAE objective in an idealised nonparametric
setting where (i) $q(\bz|\bx)$ ranges over all regular conditional
densities; (ii) $p(\bx|\bz)$ ranges over all regular conditional
densities; (iii) supports, normalisers, and disintegrations are well
defined and mutually absolutely continuous; (iv) the functional
optimum is realisable with $p_\theta(\bx)=p_{\mathrm{data}}(\bx)$ and
$q(\bz)=p(\bz)$ whenever interpreted as a generative-VAE fixed point.
Then any global optimum satisfying the unrestricted functional
first-order conditions obeys
\begin{align}
q^\star(\bz|\bx)&=p_{\theta^\star}(\bz|\bx),\label{eq:functional-forward-bayes}\\
p_{\theta^\star}(\bx|\bz)&=q^\star(\bx|\bz):=\frac{p_{\mathrm{data}}(\bx)q^\star(\bz|\bx)}{q^\star(\bz)}.\label{eq:functional-reverse-bayes}
\end{align}
Under the marginal/prior compatibility conditions, these imply
detailed balance
$p_{\mathrm{data}}(\bx)q^\star(\bz|\bx)=p(\bz)p_{\theta^\star}(\bx|\bz)$
and reversibility of the encode--decode kernel.
This is an idealised reference, not a statement about arbitrary
finite-dimensional neural-VAE stationary points.
\end{lemma}

\begin{proof}[Proof]
The encoder functional first-order condition (Lagrange-multiplier)
gives \eqref{eq:functional-forward-bayes} after normalisation; the
decoder condition gives \eqref{eq:functional-reverse-bayes} after
fixing $q$. Detailed balance and reversibility follow by combining
the two with $q^\star(\bz)=p(\bz)$.
\end{proof}

\begin{remark}[Parametric VAE caveat]\label{rem:parametric-stationarity-caveat}
In a finite neural-network VAE, vanishing parameter gradients need not
imply exact posterior matching, reverse Bayes consistency, aggregate-prior
compatibility, or reversibility. The main results
(Lemma~\ref{thm:universal}, Corollary~\ref{thm:codebook-identity},
Corollary~\ref{cor:agreement}) do not rely on Lemma~\ref{lem:joint-stationary}.
\end{remark}

\begin{proposition}[Continuous--VQ alignment spectrum]
\label{prop:spectrum}
Codebook agreement $\mathcal A$ admits a hierarchy of structural sources
indexed by which variables are promoted to variational primacy.
\textbf{Tier~1} (continuous $\beta$-VAE, Lemma~\ref{lem:joint-stationary}'s
end of the spectrum): code maps are downstream operational statistics; the
gap $1-\mathcal A$ is bounded by Corollary~\ref{cor:agreement} but is not
forced to zero.
\textbf{Tier~3} (VQ-VAE, the other end): the codebook is a primitive
variational object with explicit stationary conditions; under the sufficient
conditions of~\S\ref{app:vqvae-limit}, $\mathcal A\to 1$ as the encoder
posterior concentrates.
The continuous-VAE measurement $\bar\eta_p>0$ in \S\ref{sec:experiments}
quantifies the projection of the missing codebook stationary condition onto
Tier~1 via $\bd_c:=\bd(\bmu_c)$.
\end{proposition}

\paragraph{Variational tiers.}
The Lemma occupies one end of the hierarchy of variational primacy
formalised in Proposition~\ref{prop:spectrum}.
\textbf{Tier~1} (continuous $\beta$-VAE): vary only conditional
densities; chosen operational hard codes are a downstream object
(Corollary~\ref{thm:codebook-identity} certifies them).
\textbf{Tier~2} (GMM-VAE): also promote encoder cluster parameters
to variational variables; not evaluated here, identified as future
work in \S\ref{app:limitations-extended}.
\textbf{Tier~3} (VQ-VAE): also promote a decoupled codebook with a
commitment loss; yields $\mathcal{A}=1$ by construction
(\S\ref{app:vqvae-limit}). $\mathcal{A}<1$ in continuous VAEs reflects
the absence of explicit codebook stationary conditions; $\bar\eta_p>0$
measured in \S\ref{sec:experiments} is the projection of the missing
tier onto the Tier-1 setting via $\bd_c:=\bd(\bmu_c)$.

\subsection{Geometric reduction and sufficient-statistic alignment}
\label{app:sufficient-statistic}

\begin{proposition}[Geometric interpretation under the functional reference]
\label{prop:geometric-reduction}
Under Lemma~\ref{lem:joint-stationary} together with
(i) a $K$-component cluster structure
$p_{\mathrm{data}}(\bx)=\sum_c\pi_c p_{\mathrm{data}}(\bx|c)$,
(ii) diagonal-Gaussian aggregate conditionals
$q_c(\bz)\approx\mathcal N(\bz;\bmu_c,\bSigma_c)$,
(iii) Assumption~\ref{assume:bernoulli-decoder},
and (iv) the $0$--$1$ Bayes decision rule under these approximations:
the encoder-side rule has the Voronoi/Mahalanobis form of
Theorem~\ref{thm:encoder}, the decoder-side rule has the Type~1 Bregman
form of \S\ref{sec:decoder-codebook}, and the disagreement event
between them inherits the exact decomposition of
Corollary~\ref{thm:codebook-identity}.
\end{proposition}

\begin{proposition}[Alignment hierarchy]
\label{prop:sufficient-statistic}
Under the hypotheses of Proposition~\ref{prop:geometric-reduction},
$q_\phi(\bz)$-a.e.:
(a) hard agreement $c_{\mathrm{Bayes}}^{\mathrm{enc}}(\bz)=c_{\mathrm{Bayes}}^{\mathrm{dec}}(\bz)$
is equivalent to the two $\argmin$ rules selecting the same label;
(b) pairwise order-equivalence
$\Phi_i<\Phi_j\Leftrightarrow\Psi_i<\Psi_j$
implies (a);
(c) sufficiency of $\bd:\mathcal Z\to[0,1]^K$ for $c$ implies equal
posteriors and (under zero-tie mass) hard agreement;
(d) hard agreement alone does \emph{not} imply sufficiency; sufficiency
is recovered only if hard agreement persists under all sufficiently
small perturbations of $\{\pi_c\}$.
\end{proposition}

\begin{proposition}[Sufficient-statistic alignment]
\label{prop:suff-stat}
Under the hypotheses of Proposition~\ref{prop:sufficient-statistic}, if
the decoder map $\bd:\mathcal Z\to[0,1]^K$ is a sufficient statistic for
the latent label $c$ in the sense that
$p(c\mid\bz)=p(c\mid\bd(\bz))$ for $q_\phi(\bz)$-a.e.\ $\bz$, then
$c^{\star}_{\rm enc}(\bz)=c^{\star}_{\rm dec}(\bz)$
$q_\phi$-a.e.\ (modulo zero-tie mass), and consequently
$\mathcal A=1$ and $\bar\rho_E=0$.
\end{proposition}

\begin{corollary}[Cluster-agnostic specialization]
\label{cor:cluster-agnostic}
The conclusion $\mathcal A=1$ in Proposition~\ref{prop:suff-stat}
holds without invoking the $K$-component cluster structure of
Proposition~\ref{prop:geometric-reduction}: it is enough that
$\bd$ is sufficient for any measurable label
$c:\mathcal Z\to\{1,\dots,K\}$ that the encoder/decoder code maps
target. This makes the $\mathcal A=1$ conclusion architecture-agnostic
within the diagonal-Gaussian / product-Bernoulli class, at the cost of
the geometric (Voronoi/Bregman) interpretations of
Theorem~\ref{thm:encoder} and~\S\ref{sec:decoder-codebook}.
\end{corollary}

\begin{remark}\label{rem:alignment-hierarchy}
The affine condition $\Psi_c(\bz)=a\Phi_c(\bz)+b(\bz)$ of the main
text and the Fisher-pullback condition $G_c\propto\bSigma_c^{-1}$
(below) are progressively specialised sufficient conditions
for (b)--(c). Neither is necessary for (a)--(d) in general
architectures (e.g.\ normalising-flow decoders, full-covariance
encoders, non-Bernoulli likelihoods).
\end{remark}

\subsection{Sufficient-condition diagnostic bound: Fisher mismatch}
\label{app:diagnostic-bound}

The inverse encoder covariance $\bSigma_c^{-1}$ and the decoder
pullback Fisher metric
$G_c=J(\bmu_c)^\top\mathrm{diag}(\bd_c\odot(1-\bd_c))^{-1}J(\bmu_c)$
parametrise two local Mahalanobis-type partitions. The first-order
boundary displacement between them is diagnosed by
$\kappa_c:=\min_{a_c>0}\|\bSigma_c^{-1}-a_cG_c\|_F$ with
$a_c^\star=\langle\bSigma_c^{-1},G_c\rangle_F/\|G_c\|_F^2$.
A coordinate-invariant alternative is
$\kappa_c^{\mathrm{inv}}:=\sum_i(\log\lambda_i)^2$ with
$\{\lambda_i\}=\mathrm{eig}(\bSigma_c G_c^{-1})$.

\begin{assumption}[Cluster-separation regularity]
\label{assume:cluster-regularity}
There exist $\gamma>0$, $\rho>0$ such that for every $c\ne c'$:
(a) $\|\bd_c-\bd_{c'}\|_2\ge 2\gamma$;
(b) the second-order remainder of $\bd$ at each $\bmu_c$ is bounded by
$\rho$ on a neighbourhood of radius $\gamma/L$;
(c) the aggregate-posterior mass outside these neighbourhoods is at
most $e^{-c_0}$ for some $c_0>0$.
\end{assumption}

\begin{proposition}[Sufficient-condition diagnostic bound]
\label{thm:formal-prop30}
Under Assumptions~\ref{assume:bernoulli-decoder}, \ref{assume:cluster-regularity},
the idealised functional reference of Lemma~\ref{lem:joint-stationary},
and a first-order boundary-transport approximation,
\begin{equation}
\label{eq:formal-bound}
1-\mathcal{A}
\;\le\;
\frac{2L^2}{\gamma^2}\bar\Delta
+ C(L,\gamma,\{\pi_c\},\rho,c_0)\sum_c\pi_c\kappa_c
+ O(e^{-c_0}),
\end{equation}
with $C=\frac{1}{2\gamma^2}\cdot\max_c\frac{1}{\pi_c}\cdot(1+\rho\gamma/L)$.
The first term vanishes as $\bar\Delta\to 0$; the second vanishes when
$\bSigma_c^{-1}\propto G_c$ for every $c$. A coordinate-invariant
counterpart in $\kappa_c^{\mathrm{inv}}$ is identical in form.
Outside the stated assumptions the expression is a diagnostic, not a
certified bound.
\end{proposition}

\begin{proof}[Proof sketch]
The misclassification rate decomposes via a first-order
boundary-transport argument into (A) a gap between $c^\star_{\mathrm{enc}}$
and the pulled-back Bayes-optimal classifier under $G_c$, bounded by
$2L^2\bar\Delta/\gamma^2$ via DPI of $\bd_\#$ (Lemma analogue of
\eqref{eq:dpi-pointwise}) combined with margin/Lipschitz conversion;
plus (B) a local geometric mismatch with first-order displacement
$\|\bSigma_c^{-1}-a_c^\star G_c\|_F$, mass-weighted by $\{\pi_c\}$
with remainder $O(e^{-c_0})$ from
Assumption~\ref{assume:cluster-regularity}(c).
\end{proof}

\paragraph{Numerical reading on Setting~1.}
$L\approx\LVAL$, $\gamma\approx\GAMMAVAL$,
$\sum_c\pi_c\kappa_c\approx\KAPPAMISMATCHTERM$. Plugging
$\Delta_{\mathrm{IWAE}}\approx\EDELTA$ into~\eqref{eq:formal-bound}
gives $2L^2\bar\Delta/\gamma^2\approx 8.6\cdot 1.13$, which is
numerically large and not advertised as a non-vacuous certificate for
this checkpoint. The useful empirical reading is diagnostic: the
measured local Fisher mismatch is large, which explains why the
encoder Voronoi and decoder Bregman cells can remain misaligned even
when reconstruction loss is small.

\subsection{VQ-VAE as a Tier-3 limiting case}
\label{app:vqvae-limit}

The VQ-VAE objective~\cite{vandenoord2017neural}
\[
\mathcal{L}_{\mathrm{VQ}}
=-\log p_\theta(\bx|\bz_q(\bx))
+\|\mathrm{sg}[\bz_e(\bx)]-\be\|^2
+\beta\|\bz_e(\bx)-\mathrm{sg}[\be]\|^2
\]
admits the following relation to our framework.

\begin{proposition}[VQ-VAE encoder reduction]\label{prop:vqvae-reduction}
At a stationary point of $\mathcal{L}_{\mathrm{VQ}}$, $\Cenc$ is a
standard Voronoi diagram---Case~(i) of Theorem~\ref{thm:encoder} with
$\bSigma_c\to 0$ and (under EMA codebook updates with balanced usage)
$\pi_c=1/K$.
\end{proposition}

\paragraph{Decoder-side alignment requires an additional condition.}
$\mathcal A=1$ additionally requires
$\Psi_c(\bz)=a\Phi_c(\bz)+b(\bz)$, satisfied either (a) when $\bd_c$
approaches the one-hot indicator of $c$ (one-hot target regime), in
which case $\argmin_c D_F(\bd(\bz)\|\bd_c)$ reduces to $\argmax_c d_c(\bz)$;
or (b) under the Fisher-pullback condition of
Proposition~\ref{thm:formal-prop30} ($\kappa_c=0$).

\paragraph{Relation to Corollary~\ref{cor:agreement}.}
The deterministic VQ-VAE encoder corresponds to $\bSigma_q\to 0$, which
violates absolute continuity. A rigorous $\bar\Delta\to 0$ statement
therefore requires parametrising $q_\phi=\mathcal N(\bmu_q,\varepsilon I)$
and taking $\varepsilon\to 0^+$ while tracking $\mathcal A$ via continuity.
We conjecture that under either sufficient condition above this limit
gives $\mathcal A\to 1$ while the Bernoulli-KL bound becomes vacuous
($\bar\Delta\to 0$); a complete proof is left to future work.

\subsection{Empirical confirmation of Proposition~\ref{prop:vqvae-reduction}}
\label{app:vqvae-empirical}

We trained a VQ-VAE on MNIST $28{\times}28$ with $K{=}10$, $D{=}10$,
$\beta_{\rm commit}{=}0.25$, Adam ($10^{-3}$), $20$ epochs on a stratified
$15{,}000$-image subset; single seed, MPS backend (Apple M-series, $24.5$\,s
total wall-clock). The Setting~2 Conv-VAE backbone is shared
(\S\ref{app:experimental}); only the variational head is replaced by the
codebook quantizer with straight-through estimator. On the held-out
$10{,}000$-image test set we measure the codebook agreement
\[
  \hat{\mathcal A}_{\rm VQ}
  \;=\;
  \Pr_{\bx}\!\left[\,c^\star_{\rm enc}(\bx) \;=\; c^\star_{\rm dec}(\bx)\,\right]
  \;=\; \mathbf{1.000}
  \;\;(10{,}000/10{,}000\ \text{exact agreement}),
\]
with $c^\star_{\rm enc}(\bx)=\arg\min_k \|z_e(\bx)-\be_k\|$ (the encoder's
quantization choice) and $c^\star_{\rm dec}(\bx)=\arg\min_k D_F(\hat\bx(\bx)\,\|\,\bd_k)$,
$\bd_k=\mathrm{decoder}(\be_k)$ (the decoder's Type~1 Bregman reading at
the encoder's chosen quantization). All $K{=}10$ codebook entries remain
active on the test set; per-code utilisation ranges from $1.0\%$
(rarest) to $31.9\%$ (most frequent), with no dead codes.

This realises the prediction $\mathcal A\to 1$ of
Proposition~\ref{prop:vqvae-reduction} \emph{exactly} within numerical
precision (no estimator slack). The test is interpretable as a check on
two conditions: that the encoder's quantizer maps to distinct
$\{\be_k\}$, and that the decoder's outputs $\{\bd_k\}$ are
sufficiently separated under $D_F$ that the Bregman argmin uniquely
identifies the chosen code. A measured $\mathcal A<1$ in this protocol
would diagnose decoder collapse (some pair $\bd_i\approx \bd_j$); the
observed $\mathcal A=1.000$ rules this out for our checkpoint.
Combined with the Setting~1 Tier-1 audit
(Appendix~\ref{app:bound-illustrations}, ratio $2.71\!\times$
non-vacuous) and the MNIST Tier-1 audit (Table~\ref{tab:mnist-audited},
ratio $\sim\!160\!\times$), this realises both endpoints of the
Tier-1$\leftrightarrow$Tier-3 spectrum
(Proposition~\ref{prop:spectrum}) under the same reporting unit
$(\Ked,\mathcal A,\Reff,R,AU)$.

\begin{figure}[h]
\centering
\IfFileExists{figures/fig_vqvae_codebook.pdf}{%
\includegraphics[width=0.85\linewidth]{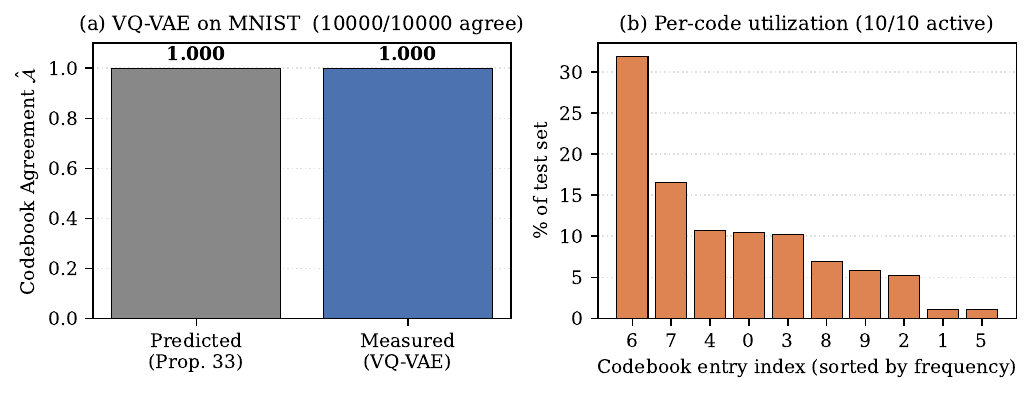}%
}{\fbox{\parbox{0.85\linewidth}{\centering\small\textit{[Figure~\ref{fig:vqvae-codebook} placeholder.]}\\[2pt]\rule{0pt}{0.6cm}}}}
\caption{\textbf{VQ-VAE codebook agreement on MNIST.}
(a) Measured $\hat{\mathcal A}_{\rm VQ}=1.000$ matches the predicted
limit of Proposition~\ref{prop:vqvae-reduction} exactly
($10{,}000/10{,}000$ test agreement). (b) Per-codebook-entry utilisation
on the $10{,}000$-image test set; all $10$ codes active. The result
realises the Tier-3 endpoint of
Proposition~\ref{prop:spectrum}; the Tier-1 endpoint is the
Conv-VAE row of Table~\ref{tab:mnist-audited}.}
\label{fig:vqvae-codebook}
\end{figure}

\section{Experimental Details}
\label{app:experimental}

This appendix covers the protocols needed to reproduce
\S\ref{sec:experiments}: the finite-grid exact certificate audit
(\S\ref{app:exact-audit}), multi-seed reporting
(\S\ref{app:multiseed-protocol}), closed-form separation stress tests
(\S\ref{app:closed-form-separations}), the four training settings
(\S\ref{app:settings}), the Setting~1-fast acceleration check
(\S\ref{app:setting1fast}), and the granularity sensitivity sweep
(\S\ref{app:gmm-sensitivity}).

\subsection{Finite-grid exact certificate audit}
\label{app:exact-audit}

For low-dimensional synthetic VAEs the safest way to turn
Corollary~\ref{cor:agreement} into a numerical certificate is to audit
a finite latent grid rather than to use an IWAE lower-bound difference.
Let $G=\{z_g\}_{g=1}^M$ be a grid with non-negative quadrature/prior
weights $w_g$ summing to one. The audited model is the discrete latent
model $p_{\theta,G}(x)=\sum_g p_\theta(x|z_g)w_g$ with
$p_{\theta,G}(g|x)\propto p_\theta(x|z_g)w_g$ and the discretised
encoder $q_{\phi,G}(g|x)\propto q_\phi(z_g|x)w_g$. All KL terms are
finite sums. This audit certifies the grid law
$(q_{\phi,G},p_{\theta,G})$; if the intended claim is about the
original continuous model, an additional quadrature error bound is
required.

\begin{algorithm}[h]
\caption{Finite-grid exact certificate audit}
\label{alg:exact-audit}
\begin{algorithmic}[1]
\Require trained $p_\theta(x|z)$, $q_\phi(z|x)$, data $\{x_n\}$,
         grid $\{(z_g,w_g)\}$, code maps $c_{\rm enc},c_{\rm dec}$
\For{each $x_n$}
\State $\ell_{ng}=\log p_\theta(x_n|z_g)+\log w_g$;\;
       $\log p_{\theta,G}(x_n)=\mathrm{logsumexp}_g\ell_{ng}$
\State $p_{ng}=\exp(\ell_{ng}-\log p_{\theta,G}(x_n))$;\;
       $q_{ng}\propto q_\phi(z_g|x_n)w_g$
\State $\Delta_G(x_n)=\sum_g q_{ng}\log(q_{ng}/p_{ng})$
\State $E_g=\mathbf 1[c_{\rm enc}(z_g)\ne c_{\rm dec}(z_g)]$;\;
       $\eta_q(x_n)=\sum_g q_{ng}E_g$;\;
       $\eta_p(x_n)=\sum_g p_{ng}E_g$
\EndFor
\State $\bar\Delta_G,\bar\eta_q,\bar\eta_p$ by sample averages
\State $B_G=d_{\rm bin}(\bar\eta_q\|\bar\eta_p)$;\;
       $\mathcal J_G=\overline{d_{\rm bin}(\eta_q(x_n)\|\eta_p(x_n))}-B_G$;\;
       $\bar\rho_G=\bar\Delta_G-B_G-\mathcal J_G$
\State validity flag: $\bar\Delta_G\ge B_G-10^{-6}$
\end{algorithmic}
\end{algorithm}

\begin{table}[h]
\centering\scriptsize
\resizebox{\linewidth}{!}{%
\begin{tabular}{@{}lcccccccc@{}}
\toprule
Dataset & seed & $\bar\Delta_G$ & $d_{\rm bin}$ & slack
& $\Aq$ & $\Ap$ & code-pair KL & valid \\
\midrule
digits & 0 & 0.2316 & 0.0010 & 0.2306 & 0.705 & 0.684 & 0.0069 & yes \\
digits & 1 & 0.2488 & 0.0032 & 0.2457 & 0.628 & 0.589 & 0.0086 & yes \\
digits & 2 & 0.2138 & 0.0014 & 0.2124 & 0.635 & 0.610 & 0.0057 & yes \\
digits & 3 & 0.3055 & 0.0009 & 0.3046 & 0.858 & 0.843 & 0.0056 & yes \\
digits & 4 & 0.2308 & 0.0026 & 0.2283 & 0.675 & 0.641 & 0.0100 & yes \\
wine & 0 & 0.1031 & 0.0018 & 0.1014 & 0.940 & 0.925 & 0.0019 & yes \\
wine & 1 & 0.1330 & 0.0034 & 0.1296 & 0.842 & 0.811 & 0.0051 & yes \\
wine & 2 & 0.1009 & 0.0002 & 0.1008 & 0.908 & 0.903 & 0.0002 & yes \\
wine & 3 & 0.1884 & 0.0008 & 0.1876 & 0.902 & 0.890 & 0.0028 & yes \\
wine & 4 & 0.1413 & 0.0105 & 0.1307 & 0.731 & 0.663 & 0.0193 & yes \\
breast cancer & 0 & 0.1551 & 0.0005 & 0.1546 & 0.904 & 0.895 & 0.0005 & yes \\
breast cancer & 1 & 0.1948 & 0.0003 & 0.1944 & 0.928 & 0.921 & 0.0005 & yes \\
breast cancer & 2 & 0.2060 & 0.0002 & 0.2058 & 0.883 & 0.876 & 0.0003 & yes \\
breast cancer & 3 & 0.1791 & 0.0009 & 0.1782 & 0.896 & 0.883 & 0.0016 & yes \\
breast cancer & 4 & 0.2184 & 0.0018 & 0.2166 & 0.901 & 0.882 & 0.0022 & yes \\
moons & 0 & 0.0056 & 0.0000 & 0.0056 & 0.500 & 0.500 & 0.0002 & yes \\
moons & 1 & 0.0058 & 0.0000 & 0.0057 & 0.517 & 0.516 & 0.0000 & yes \\
moons & 2 & 0.0063 & 0.0000 & 0.0063 & 0.516 & 0.514 & 0.0006 & yes \\
moons & 3 & 0.0048 & 0.0000 & 0.0048 & 0.578 & 0.580 & 0.0001 & yes \\
moons & 4 & 0.0056 & 0.0000 & 0.0056 & 0.500 & 0.499 & 0.0002 & yes \\
\bottomrule
\end{tabular}}
\caption{Per-seed finite-grid audit. Every row satisfies the
Bernoulli certificate $d_{\rm bin}\le\bar\Delta_G$ and the code-pair
pushforward-KL certificate for the induced grid law.}
\label{tab:exact-audit-perseed}
\end{table}

\subsection{Multi-seed reporting protocol}
\label{app:multiseed-protocol}

The released audit artifact reports five seeds per dataset (digits,
wine, breast cancer, moons). Each seed saves the full $\Ped$ table,
row-normalised $\Ked$, raw and matched agreement, off-diagonal mass,
$H(\Ce)$, $H(\Cd)$, $\Reff$, active units, rate, reconstruction loss,
the finite-grid fields of Table~\ref{tab:exact-audit-perseed}, and the
timing fields of Table~\ref{tab:timing-calibration}. Failed seeds are
included with a failure code rather than silently dropped; no seed
was dropped from the reported audit.

\subsection{Closed-form separation examples}
\label{app:closed-form-separations}

Stress tests can be generated without training. For a permutation
mismatch, set $\Ped(i,\pi(i))=1/K$ for a derangement $\pi$. For
many-to-one reading, $\Ped(i,1)=1/K$ for all $i$. For trivial collapse,
$\Ped(1,1)=1$. These prove that scalar marginal summaries do not
identify coupled communication failures and should be used as unit
tests for any implementation of $\Ked$.

\subsection{Training settings}
\label{app:settings}

\paragraph{Setting~1 (synthetic discrete-token VAE, $V{=}16$, $L{=}8$).}
Sample $z^*\sim\mathrm{GMM}_{K=10}$ in $\mathbb R^2$ on a
$\sqrt K\times\sqrt K$ lattice, isotropic component covariance
$\sigma_*^2{=}0.04I$, uniform weights. Tokeniser $\tau:\mathbb R^2\to\{1,\ldots,V\}^L$
bins $L=8$ random-orthonormal scalar projections into $V=16$ uniform-quantile
bins. Encoder input $\bx=\tau(z^*+\eta)$, $\eta\sim\mathcal N(0,\sigma_{\mathrm{tok}}^2 I)$,
$\sigma_{\mathrm{tok}}^2=0.01$. Decoder reconstructs $\bx$ via product-Categorical likelihood
(Assumption~\ref{assume:bernoulli-decoder}(b)). Encoder: token-embedding(32)
$\to$ flatten $\to$ $256^4$ $\to$ $(\bmu_q,\boldsymbol\sigma_q)$.
Decoder: $2\to 256^4\to L\cdot V$ with per-position softmax.
Adam ($\mathrm{lr}=10^{-3}$), $\beta=0.1$ constant, full-batch
$30\mathrm K$ epochs on $N=1{,}797$ samples (seed 0). Latent grid:
$300^2=90{,}000$ points spanning $[\min\bz_i-0.5,\max\bz_i+0.5]^2$.

\paragraph{Setting~1-long (legacy long-horizon timing case study).}
Same architecture and optimiser as Setting~1; $\EPOCHLONG$ epochs over
seeds $\{0,1,2\}$ with checkpoints every 500--1000 epochs.
Figure~\ref{fig:two-phase-multiseed} and
Table~\ref{tab:two-phase-multiseed} report the resulting timing case
study; the cross-dataset audit
(Table~\ref{tab:timing-calibration}) shows the plateau-then-takeoff
ordering is regime-dependent rather than universal, which motivates
the finite-horizon framing of Theorem~\ref{thm:attractor-codebook}.

\paragraph{Setting~2 (MNIST Conv-VAE).}
Encoder: $\mathrm{Conv}(1\to32,4,s2)\to\mathrm{ReLU}\to\mathrm{Conv}(32\to64,4,s2)\to\mathrm{ReLU}\to\mathrm{Flatten}\to\mathrm{FC}(64\!\cdot\!7^2\to256)\to\mathrm{ReLU}\to\mathrm{FC}(256\to 2d)$,
$d=10$. Decoder symmetric; sigmoid output; Bernoulli likelihood;
$\beta$-warmup $0\to 1$ over 5 epochs; Adam ($10^{-3}$); 20 epochs on a
stratified $15\mathrm K$ subset; $10\mathrm K$ test split for $\hat{\mathcal A}_{\mathrm{MC}}$.

\paragraph{Setting~3 (Fashion-MNIST Conv-VAE).}
Architecture identical to Setting~2; dataset substituted; single seed;
retained as a legacy visualisation check.

\paragraph{Setting~4 (CIFAR-10 extended Conv-VAE).}
Encoder adds a third $\mathrm{Conv}(64\to128,4,s2)$ stage; continuous-Bernoulli
relaxation; stratified $15\mathrm K$ subset; 20 epochs; single seed.

\subsection{Setting~1-fast: accelerated convergence variant}
\label{app:setting1fast}

Setting~1-fast is a legacy acceleration check, not central evidence.
Five modifications reduce the epoch count to reach $\hat{\mathcal A}\ge 0.95$
from $\EPOCHLONG$ to $\FASTCONV$ (a $\sim 40\times$ speed-up):
(S0) one-hot CIFAR-10 labels with a $K{=}10$ Bernoulli decoder, full
$32\!\times\!32\!\times\!3$ flattened images, class-balanced 500-per-class
subset; (S1) z-score standardisation, no PCA; (S2) AdamW with $\mathrm{lr}=5\!\times\!10^{-5}$
and weight decay $10^{-4}$; (S3) cyclical $\beta$ annealing
$\beta_{\min}=10^{-2}$ to $\approx 1$, period $\sim 250$ epochs;
(S4) gradient clipping $\|\nabla\|_2\le 10$; (S5) wider hidden dimension
(512). All other pipeline elements are identical to Setting~1.
The speed-up is consistent with three diagnostic mechanisms in
\S\ref{app:diagnostic-bound}: lower $L$ (weight decay), wider $\gamma$
(z-score), and reduced $\kappa_c$ (wider network).
$L$ is the mean spectral norm of the input--output Jacobian via
\texttt{torch.autograd.functional.jacobian} and \texttt{torch.linalg.svdvals};
$\gamma$ is the 5th percentile of per-grid-point Euclidean distance to
the nearest Type~1 Bregman bisector. The reported gap scale is an
IWAE--ELBO tightness diagnostic with $K=100$, not a certified value.

\subsection{Sensitivity to code granularity}
\label{app:gmm-sensitivity}

Codebook Agreement is defined after choosing an operational code map.
Different granularities ask different semantic questions; we sweep
unsupervised $K\in\{2,3,4,5,6\}$, refit the code maps, and recompute
deterministic mean-code agreement $\Amu$.

\begin{table}[h]
\centering\scriptsize
\resizebox{\linewidth}{!}{%
\begin{tabular}{@{}lccccc@{}}
\toprule
Dataset & $K{=}2$ & $K{=}3$ & $K{=}4$ & $K{=}5$ & $K{=}6$ \\
\midrule
digits & $0.759\pm 0.063$ & $0.773\pm 0.062$ & $0.853\pm 0.061$ & $0.877\pm 0.013$ & $0.872\pm 0.034$ \\
wine & $0.969\pm 0.016$ & $0.906\pm 0.090$ & $0.910\pm 0.092$ & $0.932\pm 0.049$ & $0.911\pm 0.073$ \\
breast cancer & $0.939\pm 0.026$ & $0.929\pm 0.021$ & $0.919\pm 0.021$ & $0.909\pm 0.029$ & $0.905\pm 0.032$ \\
moons & $0.996\pm 0.004$ & $0.973\pm 0.016$ & $0.983\pm 0.009$ & $0.986\pm 0.007$ & $0.988\pm 0.004$ \\
\bottomrule
\end{tabular}}
\caption{Sensitivity of $\Amu$ to $K$. The diagnostic is intentionally
code-map dependent; granularity should be swept or justified.}
\label{tab:k-sensitivity}
\end{table}

\paragraph{Over-clustering regime ($K\!\in\!\{8,10,12,15,20\}$).}
Table~\ref{tab:k-sensitivity} reports the low-$K$ coarsening regime on
the four-dataset audit (deterministic mean-code $\Amu$).
Figure~\ref{fig:gmm-sensitivity} complements this with the
over-clustering regime on Setting~1 (single seed, posterior-sampled
$\hat{\mathcal A}$). Unsupervised \texttt{GaussianMixture} fits on
$z_{\rm encoded}$ with $K\!\in\!\{8,10,12,15,20\}$ rebuild the decoder
codebook at each fit's component means and recompute $\hat{\mathcal A}$
on the $300\!\times\!300$ grid. The supervised baseline at the
ground-truth $K\!=\!10$ (class-conditional centroids) is
$\hat{\mathcal A}\!=\!0.849$; unsupervised fits peak nearby at
$K\!\in\!\{8,10\}$ with $\hat{\mathcal A}\!\approx\!0.88$ and degrade
smoothly toward chance as $K$ exceeds the true number of latent atoms.
The diagnostic therefore behaves as a smooth function of $K$ rather
than as a brittle estimator: it is informative under unsupervised
GMM model selection without requiring the true $K$ to be known.

\begin{figure}[h]
\centering
\IfFileExists{figures/fig_gmm_sensitivity.pdf}{%
\includegraphics[width=0.65\linewidth]{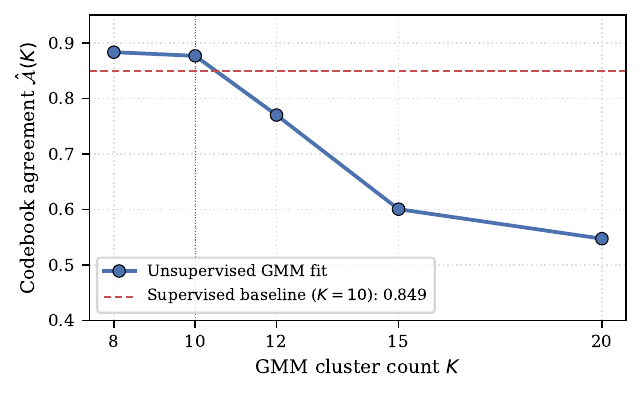}%
}{\fbox{\parbox{0.65\linewidth}{\centering\small\textit{[Figure~\ref{fig:gmm-sensitivity} placeholder.]}\\[2pt]\rule{0pt}{0.6cm}}}}
\caption{\textbf{GMM cluster-count sensitivity on Setting~1
($K\!\in\!\{8,10,12,15,20\}$).} Unsupervised \texttt{GaussianMixture}
fits on $z_{\rm encoded}$; the decoder codebook is rebuilt at each
fit's component means. The supervised baseline ($K\!=\!10$,
class-conditional centroids) is $\hat{\mathcal A}\!=\!0.849$ (dashed);
unsupervised over-clustering ($K\!\in\!\{15,20\}$) degrades
$\hat{\mathcal A}$ smoothly toward chance. This complements
Table~\ref{tab:k-sensitivity} (low-$K$ coarsening on four datasets).}
\label{fig:gmm-sensitivity}
\end{figure}

\section{Detailed Related Work}
\label{app:detailed-related}

\paragraph{Communication-theoretic mismatch and interference.}
Classical communication theory separates random channel noise from
structural receiver mismatch. Mismatched decoding studies the rate and
error consequences of decoding with a metric or channel model
different from the true channel~\cite{scarlett2020mismatched};
noisy-channel vector quantisation shows that index assignment can turn
small symbol confusions into large distortion~\cite{farvardin1990vq};
synchronisation channels study insertion/deletion/timing errors rather
than substitutions~\cite{mitzenmacher2009deletion}; and inter-symbol
interference is a temporal-memory phenomenon, not a generic synonym
for static code confusion~\cite{proakis2008digital}.
We import these distinctions carefully: static VAEs exhibit neural
\emph{inter-code interference} through the off-diagonal mass of $\Ked$,
while sequential or finite-horizon latent processes may additionally
be probed by lagged statistics such as~\eqref{eq:neural-isi}.

\paragraph{Rate--distortion and information-bottleneck views.}
Alemi et al.~\cite{alemi2018fixing} decompose ELBO into rate and
distortion; $\beta$-VAE~\cite{higgins2017beta} makes this tradeoff
explicit; the information bottleneck~\cite{tishby2000information} is a
related variational framework. Our codebook perspective is
complementary: we characterise the \emph{geometry} of the resulting
partition and couple it to the variational gap via DPI for KL
divergence~\cite{polyanskiy2024information,cover2006elements}.

\paragraph{Existing VAE diagnostics.}
Common representation diagnostics report rate, distortion, active
units~\cite{burda2016importance}, MIG and total-correlation
decompositions~\cite{chen2018isolating}, or aggregate code usage. These
are largely marginal: they can be invariant under a receiver-side
permutation or many-to-one reading. The neural codebook channel is
designed for the complementary question of \emph{coupled} encoder--decoder
confusion; Table~\ref{tab:diagnostic-stress-tests} of the main text
isolates cases where scalar diagnostics are intentionally insufficient.

\paragraph{Bregman divergences, vector quantisation, and information geometry.}
Banerjee et al.~\cite{banerjee2005clustering} establish a bijection
between exponential families and Bregman divergences;
Nielsen et al.~\cite{nielsen2010bregman} develop three types of Bregman
Voronoi diagrams via Legendre duality. We instantiate this in a trained
VAE and give the encoder--decoder reading.
VQ-VAE~\cite{vandenoord2017neural} explicitly designs a discrete
codebook with commitment; we place it at a Tier-3 endpoint
(\S\ref{app:vqvae-limit}). Lewis signaling games~\cite{lewis1969convention,lazaridou2017multi}
model emergent communication; the encoder--decoder pair is a single-agent
analog (\S\ref{app:limitations-extended}).
Arvanitidis et al.~\cite{arvanitidis2018latent} study the Riemannian
structure of VAE latent spaces; we characterise the combinatorial
(cell-complex) structure via Bregman geometry
(Proposition~\ref{prop:dually-flat}, \S\ref{app:diagnostic-bound}).

\paragraph{Latent dynamics and representation alignment.}
Fumero et al.~\cite{fumero2026latentdynamics} study latent vector
fields induced by the deterministic round trip $z\mapsto E(D(z))$.
Stochastic decoders require Markov kernels rather than deterministic
maps, so we work with the posterior-Gibbs kernels of
\S\ref{app:posterior-gibbs} and treat Fumero-style attractors as a
low-noise locally contractive limit
(\S\ref{app:fumero-relation}). The cross-model alignment program
connects to Relative Representations~\cite{moschella2022relative},
Latent Functional Maps~\cite{fumero2024latentfunctional}, and the
Platonic Representation Hypothesis~\cite{huh2024platonic}; our paper
supplies a discrete, variational-gap-controlled version for VAE code
maps.

\section{Extended Discussion, Limitations, and Outlook}
\label{app:limitations-extended}

This appendix collects the discussion material that would either
duplicate \S\ref{sec:discussion}/\S\ref{sec:limitations} of the main
text or take it off-topic. It is organised around five questions a
reviewer might raise but for which a brief answer is enough.

\paragraph{Q1: Are static VAE channel impairments literal?}
No. A static VAE does not have waveform inter-symbol interference; it
has a paired-label confusion channel $\Ked$ whose off-diagonal mass
measures inter-code interference (\S\ref{app:detailed-related}).
Mismatched decoding corresponds to $q_\phi(\bz|\bx)$ and
$p_\theta(\bz|\bx)$ inducing different operational partitions;
noisy-channel index-assignment distortion corresponds to off-diagonal
cells whose decoder prototypes are semantically far apart;
synchronisation errors correspond to ordered settings where a code
intended for one time/token/layer/horizon is read at another.

\paragraph{Q2: How does the codebook channel relate to mechanistic interpretability?}
A sparse-autoencoder dictionary~\cite{cunningham2023sparse} can be
read as an encoder-side feature proposer paired with a circuit that
implements a receiver-side codebook;
$\Ked$'s off-diagonal mass is then a direct diagnostic of feature
splitting and merging~\cite{elhage2022superposition}.
Concretely: train an SAE on layer-$\ell$ activations; define
$c_{\mathrm{enc}}^\star(h):=\arg\max_k a_k$ from SAE feature activations;
choose a (possibly many-to-one) matching $\Pi$ from features to circuit
outputs; define $c_{\mathrm{dec}}^\star(h):=\Pi^{-1}(\phi^\star(g_{\mathrm{SAE}}(f_{\mathrm{SAE}}(h))))$;
form $\hat\Ked$ on a held-out audit set; read feature splitting from
multi-modal rows and feature merging from multi-modal columns.
We do not run this experiment in the present paper; we offer the
recipe as one concrete way to instantiate $\Ked$ outside VAEs.

\paragraph{Q3: What is the relation to Lewis signaling games?}
The encoder--decoder pair can be viewed as an idealised cooperative
Lewis signaling game~\cite{lewis1969convention,lazaridou2017multi}
with states $c\in\{1,\ldots,K\}$, signals $\bz\in\mathcal Z$, actions
$\hat\bx=\bd(\bz)$, joint utility $\log p_\theta(\bx|\bz)$, and channel
cost $D_{\KL}(q_\phi(\bz|\bx)\|p(\bz))$; the negative $\beta$-VAE
ELBO is the cost-augmented objective. Under the symbol correspondence
$\bz\leftrightarrow m$, $q_\phi\leftrightarrow S_\phi$,
$p_\theta\leftrightarrow R_\theta$, this is identical to the
$\beta$-VAE ELBO. Codebook Agreement therefore provides a geometric
compatibility measure for this idealised abstraction, not a complete
equilibrium characterisation. A recent loss decomposition into
information loss and co-adaptation loss admits a structural parallel
with Corollary~\ref{thm:codebook-identity}---$\mathcal{L}_{\mathrm{info}}\!\leftrightarrow\!d_{\mathrm{bin}}$,
$\mathcal{L}_{\mathrm{adapt}}\!\leftrightarrow\!\mathcal{J}+\bar\rho$,
co-adaptation overfitting $\!\leftrightarrow\!\sum_c\pi_c\kappa_c$---but
the correspondences are qualitative (population-level KL chain rule
vs.\ finite-sample train/test gap; continuous $\bz$ vs.\ discrete $m$).

\paragraph{Q4: Is high agreement always good?}
No. The framework predicts that posterior collapse appears as high
apparent agreement over a low-entropy, low-information code, with
small $I(X;\Ce)$, low decoder sensitivity
$\E\|\nabla_z\log p_\theta(\bx|\bz)\|^2$, and low effective decoder
Fisher rank. Codebook Agreement is therefore uninformative
stand-alone; it should be paired with code entropy, $I(X;\Ce)$,
rate, finite-horizon multiplicity, and decoder-rank diagnostics.
Pre-trained encoder--decoder pairs such as Optimus
(BERT$\to$GPT-2)~\cite{li2020optimus}---where the autoregressive
Categorical decoder makes Assumption~\ref{assume:bernoulli-decoder}(b)
applicable position-wise---are a natural future-work target.

\paragraph{Q5: What are the key open directions?}
Beyond the four points in \S\ref{sec:limitations} we flag four:
\begin{enumerate}[label=(L5.\arabic*),itemsep=2pt,topsep=2pt,leftmargin=3em]
\item \textbf{Other specialisations of Lemma~\ref{thm:universal}:}
the universal identity supplies a menu indexed by the post-processing
$T$. Beyond binary disagreement, $K$-ary cluster-assignment $T$ yields
categorical-KL bounds; encoder-mean $T$ yields Gaussian-KL bounds;
threshold indicators yield concentration-based bounds.
\item \textbf{Decoder families beyond product-Bernoulli:}
Proposition~\ref{prop:dually-flat} is a first-order local realisation
specific to the diagonal-Gaussian / product-Bernoulli class;
Proposition~\ref{prop:suff-stat} and Corollary~\ref{cor:cluster-agnostic}
are architecture-independent starting points for adapting to
continuous Gaussian, mixture, or normalising-flow decoders.
\item \textbf{Cross-model latent-vector-field alignment:}
\S\ref{app:finite-horizon-cross} gives the formulation; the two-model
experiment (train two VAEs with different seeds, compute $R$ from
shared anchors or a latent functional map, report
$\mathcal A^{\mathrm{cross}}_H(R)$ and $\mathfrak E_{\mathrm{vf}}(R)$
for $H\in\{0,1,2,4,8\}$) is left for future work.
\item \textbf{Tier-2 GMM-VAE and tempered $\beta\ne 1$ reference:}
the Tier-2 model promoting only encoder cluster parameters
(\S\ref{app:functional-reference}) was not evaluated; a fully formal
$\beta\ne 1$ functional-reference analogue requires the tempered
posterior $p_{\theta,\beta}(\bz|\bx)\propto p(\bz)p_\theta(\bx|\bz)^{1/\beta}$
and is also future work.
\end{enumerate}

\paragraph{Differentiation from prior encoder--decoder consistency work.}
Three lines of prior work address related encoder--decoder questions but
ask materially different questions and use materially different objects.
Cemgil et al.~\cite{cemgil2020avae} (NeurIPS 2020) ask whether a trained
VAE consistently encodes typical samples generated from its decoder, and
answer ``no''; their fix is a self-consistency objective implemented
through a Markov chain that alternates encoder and decoder. Wang et
al.~\cite{wang2020coupled} formalize encoder--decoder incompatibility via
differential geometry of chart maps and propose a Coupled-VAE objective
that ties a deterministic-autoencoder branch to the VAE branch. Dang
et al.~\cite{dang2024posterior} (ICLR 2024) develop posterior-collapse
detection theory for conditional and hierarchical VAEs from the
\emph{marginal} side. Our framework is complementary to all three: rather
than (a) propose a training-time fix
(\cite{cemgil2020avae,wang2020coupled}) or (b) detect a marginal-side
failure mode (\cite{dang2024posterior}), we (c) construct a
\emph{coupled-channel finite-codebook diagnostic}
($\Ked,\mathcal A,\Reff$) whose off-diagonal mass receives a
one-dimensional Bernoulli-KL audit certificate via the variational gap
(Corollary~\ref{cor:agreement}). Proposition~\ref{prop:marginal-impossibility}
delineates exactly the diagnostic gap that motivates this coupled view,
and the operational reporting unit
$(\Ked,\mathcal A,\Reff,R,\mathrm{AU})$ is downstream-task agnostic by
construction. None of the three prior frameworks supplies a
finite-codebook channel diagnostic with this property.

\section{Practitioner Notes}
\label{app:practitioner-notes}

This appendix records implementation conventions and failure modes
that are easy to miss when reproducing the diagnostics. None of these
notes adds a theoretical claim; they are intended to make the main
identities harder to mis-implement.

\begin{remark}[Encoder/decoder $x$-consistency requirement]
\label{rem:x-consistency}
The certificate of Corollary~\ref{cor:agreement} is built around the
true model posterior $p_\theta(z|x)\propto p_\theta(x|z)p(z)$, so
\emph{the encoder input and the decoder output must be the same $x$}.
This rules out conditional models in which the encoder consumes one
variable (e.g.\ an image) and the decoder produces another (e.g.\ a
class label, a caption, or an action). Such image-to-label or
multimodal-conditional models are out of scope for the audit identity:
the relevant variational gap is well-defined, but the disagreement
event $E=\{c^{\star}_{\rm enc}\ne c^{\star}_{\rm dec}\}$ does not couple
to the same posterior, so the Bernoulli-KL bound does not apply.
\end{remark}

\paragraph{Bernoulli geometry on continuous-valued pixels.}
For $x\in[0,1]^K$, $\prod_k d_k^{x_k}(1-d_k)^{1-x_k}$ is not normalised
unless $x$ is binary. The correct continuous-Bernoulli density
includes a normaliser $C(\lambda)$ depending on the decoder parameter,
not on $x$~\cite{loaiza2019continuous}. Dropping this is therefore not
an additive constant in $\theta$. In this paper, the Bernoulli/BCE
construction should be read as a surrogate reconstruction geometry
for grey-scale pixels unless the continuous-Bernoulli normaliser is
explicitly included; this is harmless for Setting~1's one-hot labels
but matters when interpreting MNIST or CIFAR-10 as likelihood models.

\paragraph{Per-example versus aggregate gaps.}
$\bar\Delta$ denotes the per-example mean true variational gap.
The independent identity-verification check also reports the aggregate
batch sum $\Delta_{\mathrm{agg}}\approx\EDELTAAGG$ over
$|\mathcal B|=\CHECKBATCHSIZE$ examples, with aggregate residual
$\RESIDUALAGG$ nats. The relation
$\Delta_{\mathrm{agg}}/|\mathcal B|\approx 1.15\approx\EDELTA$
explains the apparent scale difference; the binary-KL bound uses the
per-example convention.

\paragraph{High-dimensional latent code maps.}
The Voronoi and Bregman instantiations of \S\ref{sec:theory} are most
natural in low-dimensional latent spaces where the cell complexes are
visualisable and the GMM summary of the aggregate posterior is reliable.
For larger $d$ (e.g.\ $d\!\geq\!32$ for transformer residual streams or
SAE features), the framework still applies since
Lemma~\ref{thm:universal} and Corollary~\ref{cor:agreement} require only
measurability of the code maps. Practical instantiations in this regime
include: (a) $k$-means or mini-batch $k$-means on the aggregate posterior
samples, with $c^\star_{\rm enc}(z)$ the nearest cluster index; (b) GMM
fit after a PCA projection to a manageable dimension, restoring the
weighted-Voronoi geometry of Theorem~\ref{thm:encoder} on the projected
subspace; (c) for SAE-style applications, the top-$1$ active feature
indicator $c^\star_{\rm enc}(h)=\arg\max_k a_k$ and a matched-circuit
decoder readout $c^\star_{\rm dec}$ as sketched in
Appendix~\ref{app:limitations-extended}~Q2. None of these changes the
audit identity (\ref{eq:main-identity}); they change only the geometric
interpretation of which $E$ the certificate bounds, and hence what
$\bar\rho_E$ is sensitive to. The portability of $\Ked$ across
dimensions is, in this sense, exactly the portability of choosing a
measurable hard-code map.

\paragraph{Cluster count $K$.}
Main runs set $K$ to the known class count---an evaluation protocol,
not an inference method for the true number of latent atoms.
\S\ref{app:gmm-sensitivity} reports a sensitivity sweep; in unlabelled
settings, $K$ should be selected by an explicit model-selection rule
or replaced by a non-parametric clustering diagnostic.

\paragraph{Determinism and seeds.}
Qualitative experiments are single-seed illustrations with fixed
seeds, deterministic data splits, and the same latent grid so that
changes in $\hat{\mathcal A}(t)$ reflect the diagnostic rather than
sampling noise. This is reproducibility, not statistical robustness;
multi-seed reporting is the appropriate target for empirical extensions
and is what the cross-dataset audit (\S\ref{app:exact-audit}) provides.

\paragraph{Bisection for the binary-KL bound.}
The inversion $d_{\mathrm{bin}}(p\|q)=\bar\Delta$ is numerically
stable only on the correct monotone branch; restrict to $p\in[q,1)$
since the bound is on the upper tail. Implementations should clip
neither $p$ nor $q$ outside $[0,1]$.

\paragraph{Computing the neural codebook channel.}
Given samples $\{\bz_n\}_{n=1}^N$ from $\bar q$, form paired labels
$(i_n,j_n)=(c^\star_{\mathrm{enc}}(\bz_n),c^\star_{\mathrm{dec}}(\bz_n))$,
the normalised contingency $\hat\Ped$, and the row-normalised channel
$\hat\Ked$. Report at minimum diagonal mass $\hat{\mathcal A}$,
off-diagonal mass $1-\hat{\mathcal A}$, row entropy, column
concentration, and either $\hat\Reff=I_{\hat\Ped}(\Ce;\Cd)$ in nats
or the explicitly normalised $\hat\Reff/\log K$. If labels are
arbitrary, align by a fixed prototype rule or report both raw and
optimal-permutation diagonals.

\paragraph{Agreement is not emergence.}
A model can have high $\mathcal A$ because both encoder and decoder
collapse to one effective basin. Agreement is therefore a disagreement
probe, not a stand-alone representation-quality metric, and should be
paired with the non-collapse diagnostics flagged in Q4 above.

\section{Additional Figures}
\label{app:additional-figures}

This appendix collects supplementary figures referenced from the main
text and from earlier appendices: the full long-run channel sequence,
cross-dataset visualisations, the four Bregman Voronoi types, the
boundary-overlay reading, the affine and Legendre-dual reformulations,
and the two-batch identity verification scatter.

\begin{figure}[h]
\centering
\IfFileExists{figures/fig_codebook_channel_heatmaps_full.png}{%
\includegraphics[width=\linewidth]{figures/fig_codebook_channel_heatmaps_full.png}%
}{\IfFileExists{figures/fig_codebook_channel_heatmaps_full.pdf}{%
\includegraphics[width=\linewidth]{figures/fig_codebook_channel_heatmaps_full.pdf}%
}{\IfFileExists{figures/fig_codebook_channel_heatmaps_main.png}{%
\includegraphics[width=\linewidth]{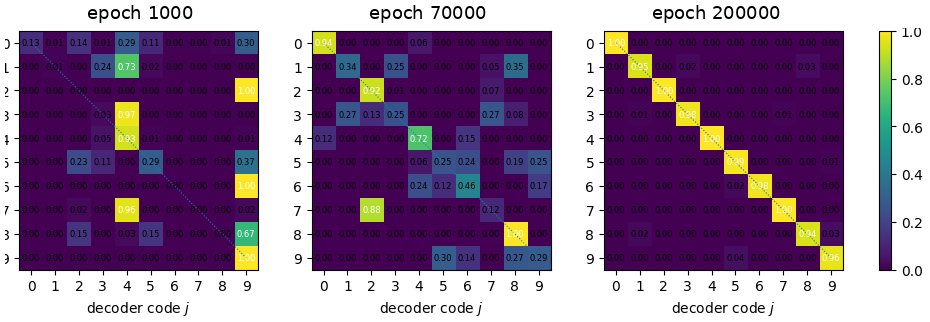}%
}{%
\fbox{\parbox{0.85\linewidth}{\centering\small\textit{[Figure~\ref{fig:codebook-channel-heatmaps-full} placeholder.]}\\[2pt]\rule{0pt}{0.6cm}}}%
}}}
\caption{Checkpoint sequence for the neural codebook channel.
Annotated entries are row-normalised probabilities of $\Ked(j|i)$; the
diagonal guide marks raw encoder--decoder label agreement.
When the 5-panel asset is available the rendered checkpoints are
1000, 30000, 70000, 130000, 200000 and the corresponding
$\Reff/\log K$ values are 0.290, 0.160, 0.560, 0.852, 0.959; with the
3-panel asset the checkpoints are 1000, 70000, 200000 with values
0.290, 0.560, 0.959.}
\label{fig:codebook-channel-heatmaps-full}
\end{figure}

\begin{figure}[h]
\centering\small
\begin{tabular}{@{}lccp{0.34\textwidth}@{}}
\toprule
Setting & $\hat{\mathcal A}$ & ARI & Role \\
\midrule
MNIST Conv-VAE & $\HATATWO\%$ & $\ARITWO$ & higher-dimensional check \\
Fashion-MNIST Conv-VAE & $\HATAFMNIST\%$ & $\ARIFMNIST$ & harder image classes; single seed \\
CIFAR-10 Conv-VAE & $\HATACIFAR\%$ & $\ARICIFAR$ & continuous-Bernoulli sanity check \\
\bottomrule
\end{tabular}
\caption{Higher-dimensional visualisation summary for Settings~2--4.
Retained for breadth; main reviewer-facing robustness evidence is the
multi-seed audit in Table~\ref{tab:cross-dataset-summary}.}
\label{fig:cross-dataset}
\end{figure}

\begin{figure}[h]
\centering\small
\begin{tabular}{@{}lp{0.62\textwidth}@{}}
\toprule
Panel concept & What it shows \\
\midrule
Encoder Euclidean / weighted Voronoi & aggregate-posterior assignment regions \\
Decoder Type~1 Bregman Voronoi & decoder-side assignment pulled back through $\bd(\bz)$ \\
Decoder Type~2 dual & equivalent dual construction in gradient/logit coordinates \\
Decoder Type~3 symmetrised & effect of symmetrising the asymmetric Bregman divergence \\
\bottomrule
\end{tabular}
\caption{Encoder and decoder codebooks compared. The encoder and
decoder induce different cell complexes; their disagreement is
$1-\mathcal A$.}
\label{fig:four-types}
\end{figure}

\begin{figure}[h]
\centering\small
\begin{tabular}{@{}lp{0.66\textwidth}@{}}
\toprule
Boundary object & Interpretation \\
\midrule
Encoder boundary & two aggregate-posterior Mahalanobis energies tied \\
Decoder boundary & two Bregman reconstruction prototypes tied \\
Miscommunication zone & latent points where these two induced labels disagree \\
\bottomrule
\end{tabular}
\caption{The gap between encoder and decoder boundaries is the
geometric support of the disagreement event $E$.}
\label{fig:boundary-overlay}
\end{figure}

\begin{figure}[h]
\centering\small
\begin{tabular}{@{}llc@{}}
\toprule
Construction & Compared against & Agreement \\
\midrule
Affine reformulation & direct Type~1 Bregman diagram & $100.00\%$ \\
Weighted Euclidean rep. & direct Type~1 Bregman diagram & $99.97\%$ \\
Type~2 (dual) & Type~1 of $F^*$ in gradient space & $100.0\%$ \\
\bottomrule
\end{tabular}
\caption{Numerical verification of the Bregman reformulations
(Propositions~\ref{prop:affine-reformulation}--\ref{prop:legendre-dual}).}
\label{fig:bregman-reformulation}
\end{figure}

\begin{figure}[h]
\centering\small
\begin{tabular}{@{}llc@{}}
\toprule
Check & Value & Interpretation \\
\midrule
Identity residual & $\RESIDUALAGG$ nats & aggregate numerical residual \\
Reference scale & $\Delta_{\mathrm{agg}}\approx\EDELTAAGG$ nats & aggregate IWAE--ELBO tightness \\
Relative residual & $1.8\%$ & estimator-level consistency check \\
\bottomrule
\end{tabular}
\caption{Two-batch independent check of the
Corollary~\ref{thm:codebook-identity} estimator. The proof establishes
the identity (\S\ref{app:proof-codebook-specialization}); this table
is a numerical consistency check on the estimator
(\S\ref{app:empirical-estimator}).}
\label{fig:identity-verification}
\end{figure}


\section*{NeurIPS Paper Checklist}

\begin{enumerate}
\item \textbf{Claims.} Q: Do the main claims of the paper accurately reflect the contributions and scope? A: [Yes] Justification: The abstract and introduction state a Theory contribution: the coupled diagnostic object $\Ked$, the universal post-processing decomposition (Lemma~\ref{thm:universal}), the binary disagreement specialization and Bernoulli-KL certificate (Corollaries~\ref{thm:codebook-identity}--\ref{cor:agreement}), and the marginal-insufficiency separation (Proposition~\ref{prop:marginal-impossibility}). The experiments are described as finite-grid audits and diagnostic robustness checks, not as benchmark or training-dynamics claims.

\item \textbf{Limitations.} Q: Does the paper discuss limitations? A: [Yes] Justification: Section~\ref{sec:limitations} states that code maps are researcher-specified operational statistics; empirical certificates require exact enumeration or controlled quadrature; the finite-grid audits certify induced grid laws only; high agreement does not by itself imply non-collapsed emergence; the Bregman/Fisher analysis is model-class-specific; and no universal plateau-then-takeoff training law is claimed. Appendix~\ref{app:limitations-extended} gives extended limitations and future work.

\item \textbf{Theoretical results.} Q: Are assumptions and proofs provided for all theoretical results? A: [Yes] Justification: Lemma~\ref{thm:universal} is proved in Appendix~\ref{app:proof-universal}; Corollaries~\ref{thm:codebook-identity} and~\ref{cor:agreement} are proved in Appendices~\ref{app:proof-codebook-specialization} and~\ref{app:cor18-proof}; Proposition~\ref{prop:marginal-impossibility} is proved in the main text; Theorem~\ref{thm:encoder} and the model-class-specific geometric diagnostics are proved or derived in Appendices~\ref{app:encoder-proof}--\ref{app:diagnostic-bound}; and finite-horizon dynamics statements are isolated in Appendix~\ref{app:posterior-gibbs}. Assumptions such as measurability, absolute continuity, standard Borel spaces, finite grids, regularity, and boundary conditions are stated where used.

\item \textbf{Experimental reproducibility.} Q: Does the paper fully disclose all the information needed to reproduce the main experimental results? A: [Yes] Justification: Section~\ref{sec:experiments} and Appendix~\ref{app:experimental} specify datasets, architecture, optimizer, learning rate, batch size, epochs, seeds, grid size, code-map construction, and reporting protocol. The intended anonymized supplementary code archive regenerates the diagnostic, finite-grid, sensitivity, marginal-insufficiency, and timing tables.

\item \textbf{Open access to data and code.} A: [Yes] Justification: The datasets used in the main audits are standard public sklearn datasets or synthetic two-moons data. The submission is intended to include anonymized supplementary code for review and de-anonymized code after acceptance.

\item \textbf{Experimental setting.} A: [Yes] Justification: Section~\ref{sec:experiments} and Appendix~\ref{app:experimental} describe the four-dataset, five-seed audit, the $800$-epoch training schedule, checkpoint cadence, and $41\times41$ finite-grid posterior evaluation.

\item \textbf{Experiment statistical significance.} A: [Yes] Justification: Main diagnostic summaries are reported as mean $\pm$ standard deviation over five seeds per dataset. The paper does not use the experiments to claim superiority over baselines; they are reproducibility and calibration checks for a theory-first diagnostic.

\item \textbf{Compute resources.} A: [Yes] Justification: The main finite-grid audits are low-dimensional and use four small public/synthetic datasets over five seeds. Appendix~\ref{app:experimental} reports the training and audit protocol. Legacy long-horizon trajectory experiments are kept in the appendix only as illustrations and are not central evidence.

\item \textbf{Code of ethics.} A: [Yes] Justification: The work is a diagnostic/theoretical study using public or synthetic datasets and does not introduce a deployed system, human-subject data collection, or dual-use capability.

\item \textbf{Broader impacts.} A: [N/A] Justification: The direct contribution is a diagnostic and theoretical framework for representation analysis. Potential downstream impact is methodological: practitioners may avoid over-interpreting latent usage as shared meaning. No direct societal deployment is proposed.

\item \textbf{Safeguards.} A: [N/A] Justification: No model release with foreseeable deployment risk is proposed; the code artifact supports reproduction of small-scale diagnostics.

\item \textbf{Licenses.} A: [Yes] Justification: The main datasets are public sklearn datasets or synthetic data; implementation dependencies should be listed in the supplementary repository.

\item \textbf{Assets.} A: [Yes] Justification: Any released code, generated tables, and figures should be included under an explicit repository license.

\item \textbf{Crowdsourcing / human subjects.} A: [N/A] Justification: No crowdsourcing or human-subject study is used.

\item \textbf{IRB approvals.} A: [N/A] Justification: No human-subject data or intervention is used.

\item \textbf{LLM usage.} A: [N/A] Editing (e.g., grammar, spelling, word choice)
\end{enumerate}

\end{document}